\documentclass[11pt, a4paper, goog]{google}

\usepackage[authoryear, sort&compress, round]{natbib}
\usepackage{bbm}
\usepackage{algorithm}
\usepackage{algpseudocode}
\usepackage{tcolorbox}
\tcbuselibrary{breakable}

\newtheorem{proposition}{Proposition}

\providecommand{\E}{\mathbb{E}}
\providecommand{\mc}[1]{\mathcal{#1}}
\providecommand{\set}[1]{{\{ #1 \}}}
\providecommand{\abs}[1]{{| #1 |}}
\providecommand{\brac}[1]{{[ #1 ]}}
\providecommand{\bracBig}[1]{{\Big[ #1 \Big]}}
\providecommand{\parenBig}[1]{{\Big( #1 \Big)}}
\providecommand{\bsl}{\backslash}
\providecommand{\Eqref}[1]{Equation~\ref{#1}}


\uselogo{}

\title{Acceleration of Diffusion Language Model through Discrete Average Generator}

\correspondingauthor{yidongouyang@g.ucla.edu}

\reportnumber{}

\author[1,2,*]{Yidong Ouyang}
\author[1]{Zhengyan Wan}
\author[2]{Themis Haris}
\author[2]{Tian Tan}
\author[2]{Liqian Peng}
\author[2]{Henry Li}
\author[2]{Ziqian Lin}
\author[2]{Jianhang Chen}
\author[2]{Maryam Karimzadehgan}
\author[2]{Alec Go}
\author[1]{George Michailidis}

\affil[1]{University of California, Los Angeles}
\affil[2]{Google}
\affil[*]{Work done as a research intern at Google}

\begin{abstract}
Discrete diffusion models and flow matching have emerged as powerful frameworks
for generative modeling over discrete state spaces, yet efficient few-step generation
remains a fundamental challenge. In this work, we introduce the \emph{Discrete Average
Generator}, a principled extension of MeanFlow to
Continuous-Time Markov Chains (CTMCs). Analogously to how MeanFlow defines an average
velocity field over a time interval in continuous spaces, we define an average generator as the normalized increment of the transition kernel over a time interval. We show that this average generator satisfies a self-consistency identity, which provides the foundation for our training objective. We further develop training strategies that align with the standard training paradigm of diffusion language models while keeping the resulting objective tractable. When projected onto per-coordinate marginals, the self-consistency identity admits a closed-form expression, enabling efficient training and inference. In Potts model simulations, our objective reduces the total variation distance of the $K$-step sampler by up to 67\%. On OpenWebText, our method achieves the lowest generative perplexity among the evaluated methods for 8 to 64 sampling steps while enabling a $16\times$ acceleration, and achieves comparable performance to existing methods on ImageNet.

\end{abstract}

\begin{document}

\maketitle

\section{Introduction}

Diffusion models and flow matching have achieved remarkable success in continuous
domains such as image and audio generation~\citep{lipman2024flow}.
Extending these generative frameworks to discrete spaces, such as natural language,
has attracted growing interest, with discrete flow matching and masked diffusion
models emerging as promising approaches~\citep{campbell2024generative,
gat2024discrete, shi2024simplified, sahoo2024simple}.
However, a fundamental challenge shared with their continuous counterparts remains:
sampling requires multiple sequential network evaluations, making inference
computationally expensive. The difficulty is compounded in discrete spaces, where practical
samplers update all coordinates independently within a step: with few steps, each
coordinate is drawn from a law computed from the instantaneous rate rather than from the
finite-step kernel, and the dependence between coordinates updated together is lost.

In continuous settings, this bottleneck has been addressed by consistency
models~\citep{Song2023ConsistencyM}, progressive distillation~\citep{salimans2022progressive},
and more recently MeanFlow~\citep{Geng2025MeanFF}, which defines an average velocity
field over time intervals and enables high-quality generation in very few steps without
multi-step distillation. The key insight of MeanFlow is that the average velocity satisfies
a self-consistent identity involving the instantaneous velocity and its material derivative,
which can be exploited as a fixed-point training objective.

In this work, we propose the \emph{Discrete Average Generator}, a principled extension
of MeanFlow to discrete state spaces governed by Continuous-Time Markov Chains (CTMCs).
The central challenge is that the continuous notion of velocity and
its spatial gradient has no direct analogue in discrete spaces. We resolve this by defining the CTMC average generator over an interval $[t,r]$ as the normalized increment of the transition kernel, i.e.\ the unique object for which a single Euler step of size $r-t$ reproduces the exact transition law. We further derive an average generator identity that forms the basis of our training objective. Inspired by the progress of velocity prediction ($u$-prediction / $v$-prediction) to data prediction ($x$-prediction) made in continuous MeanFlow literature described in Section \ref{ap:meanflow-bg}, we develop a corresponding progression from transition rate prediction ($U$-prediction / $Q$-prediction) to posterior prediction ($\tilde{p}$-prediction) in Section \ref{sec:objective}.

Although the $\tilde{p}$-prediction is compatible with existing pretrained diffusion language models, it is intractable at the sequence level because of the huge computational cost associated with the multiplication of the generator matrix with dimension
$|\mathcal{S}|^{\mathcal{D}} \times |\mathcal{S}|^{\mathcal{D}}$.
By projecting the identity onto per-coordinate marginals, as described in Section~\ref{sec:tractable}, we obtain a tractable closed-form involving only $\mathcal{D} \times |\mathcal{S}|$-dimensional quantities.
Our main contributions are as follows:
\begin{itemize}
    \item We introduce the Average Generator for CTMCs and prove the Discrete Average
    Generator Identity, the discrete analogue of the
    MeanFlow identity.
    \item We derive a tractable training objective based on $\tilde{p}$-prediction and per-coordinate marginal projection, which aligns with existing diffusion language model parameterizations and admits a closed-form expression involving only $\mathcal{D} \times |\mathcal{S}|$-dimensional quantities.
    \item We validate our approach through simulations, text generation on the OpenWebText dataset, and image generation on the ImageNet dataset. On OpenWebText, our method achieves the lowest generative perplexity among the compared methods for $8$ to $64$ sampling steps and maintains lower generative perplexity with  $16\times$ fewer generation steps against the strongest baseline. On ImageNet, our method achieves comparable FID and Inception Score to existing acceleration methods for diffusion language models.
\end{itemize}

\section{Preliminary}\label{sec:prel}

\subsection{Discrete Flow Matching}
Consider a Continuous-Time Markov Chain (CTMC) defined on the state space $\mc{S}^{\mc{D}}$ with marginal probability mass functions $\set{p_t}_{t\in[0,1]}$. Let $Q_t$ denote the associated transition rate matrix (i.e., the generator) at time $t$, which has dimensions $|\mc{S}|^{\mc{D}} \times |\mc{S}|^{\mc{D}}$.
For any distinct states $x, z \in \mc{S}^{\mc{D}}$ (where $x \neq z$), the entry $Q_t(x,z)$ represents the instantaneous transition intensity from state $x$ to state $z$ at time $t$.
If $Q_t(x,z)$ satisfies the following rate-properties:
\begin{align*}
    Q_t(x,z)\ge 0, \text{ for any }z\neq x, \text{ and }~  \sum_{z\in\mc{S}^\mc{D}}Q_t(x,z)=0,
\end{align*}

and the Kolmogorov forward equation (also known as the continuity equation):
\begin{equation}\label{eq:Kolmogorov}
\dot{p}_t(x)=\sum_{z\in\mc{S}^\mc{D}}Q_t(z,x)p_t(z)=\sum_{z\neq x}Q_t(z,x)p_t(z)-\sum_{z\neq x}Q_t(x,z)p_t(x),
\end{equation}
the transition rate $Q_t$ is also the generator of the CTMC $\set{p_t}_{t\in[0,1]}$. In
particular, $p_{t+h|t}(x|z)=\delta_{z}(x)+Q_t(z,x)h+o(h)$ for a small time step $h$.

To learn a transition rate $Q_t$ that can transport from an initial noise distribution $p_0$ to a target data distribution $p_1$, a well-developed method is to learn the conditional expectation of the conditional transition rate $Q_t(x,z|x_1)$ that generates the conditional probability path $p_{t|1}(\cdot|x_1)$. Then $Q_t(x,z)=\E_{p_{1|t}(x_1|x)}\brac{Q_t(x,z|x_1)}$ can generate the target probability path $p_t$ \citep[see Proposition 3.1 of][]{campbell2024generative}, i.e., it satisfies the Kolmogorov forward \Eqref{eq:Kolmogorov}. For completeness, we include the proof in Appendix \ref{appendix:prop}. We are free to define the conditional probability path $p_{t|1}(\cdot|x_1)$ and with the corresponding conditional transition rate as needed.

A common approach is to construct a coordinate-wise conditional probability
path and transition rate,
\begin{equation}\label{eq:coordinate-wise conditional path}
   p_{t|1}(x|x_1)=\prod_{d=1}^{\mc{D}}p^d_{t|1}(x^d|x_1^d),\qquad Q_t(x,z|x_1)=\sum_{d=1}^\mc{D}\delta_{x^{\bsl d}}(z^{\bsl d})Q_t^{d}(x^d,z^d|x_1^d),
\end{equation}
where we use the notation
$x^{\bsl d}$ to denote the vector obtained by
removing the $d$-th coordinate of $x$, i.e.,
$x^{\bsl d} := (x^1, \ldots, x^{d-1}, x^{d+1}, \ldots, x^\mathcal{D})$, $Q_t^{d}(x^d,z^d|x_1^d)$ is the conditional transition rate that generates the conditional probability path $p_{t|1}^d$, and $\delta_x(z)$ the Kronecker delta satisfying $\delta_x(z)=1$ if $x=z$ and $\delta_x(z)=0$ if $x\neq z$.
A popular choice of probability path and the associated conditional transition rate used in the previous works \citep{campbell2024generative,gat2024discrete} is
\begin{equation}\label{eq:mixture path}
    p_{t|1}^d(x^d|x_1^d)=(1-\kappa_t)p_0^d(x^d)+\kappa_t\delta_{x_1^d}(x^d),\qquad Q_t^{d}(x^d,z^d|x_1^d)=\lambda_t\big(\delta_{x_1^d}(z^d)-\delta_{x^d}(z^d)\big),
\end{equation}
where $\kappa_t:[0,1]\to[0,1]$ is a non-decreasing function satisfying $\kappa_0=0$ and
$\kappa_1=1$, and $\lambda_t\triangleq\frac{\dot{\kappa}_t}{1-\kappa_t}$. Common sources are
the uniform distribution $p_0^d=\mathrm{Unif}(\mathcal{S})$ and the masked source
$p_0^d=\delta_{\mathbf{m}}$, which puts all mass on a mask token $\mathbf{m}$.
After defining the conditional path and rate, the marginal transition rate is given by
\begin{equation}\label{eq:analytic_rate}
\begin{aligned}
Q_t(x,z)=\sum_{x_1}Q_t(x,z|x_1)p_{1|t}(x_1|x)=&~\sum_{d=1}^\mc{D}\delta_{x^{\bsl d}}(z^{\bsl d})\sum_{x_1^d}Q_t^{d}(x^d,z^d|x_1^d)p^d_{1|t}(x_1^d|x) \\
    \overset{\triangle}{=}&~\sum_{d=1}^\mc{D}\delta_{x^{\bsl d}}(z^{\bsl d})Q_t^{d}(x,z^d),\notag
\end{aligned}
\end{equation}
where $p_{1|t}^d(x_1^d|x)=\sum_{x^{\bsl d}_1}p_{1|t}(x_1|x)$ is the posterior. When the
coordinate is clear from the argument we drop the superscript and write $p_{1|t}(z^d|x)$
for $p^d_{1|t}(z^d|x)$; for the mixture path, $Q^d_t(x,z^d)=\lambda_t\,p_{1|t}(z^d|x)$ for
$z^d\neq x^d$.

\paragraph{Training objective}
To learn the transition rate, a common approach is to minimize the Bregman divergence between the conditional generator and the estimated generator \citep{holderrieth2024generator,lipman2024flow,shaul2024flow,wan2025error}:
\begin{equation}\label{eq:training-rate}
    \E_{x_1,t,x_t\sim p_{t \mid 1}(x_t\mid x_1)}\bracBig{\sum_{z\neq x_t}D_F\parenBig{Q_t(x_t,z|x_1)\Big\|Q_t(x_t,z)}},\notag
\end{equation}
where $D_F$ denotes the Bregman divergence
induced by $F(x) = x\log x$, $D_F(a\|b)\triangleq F(a)-F(b)-F'(b)(a-b)=a\log\frac{a}{b}-a+b$,
which is defined for $a\ge0$ and $b>0$ and vanishes if and only if $a=b$.
After some derivation in Appendix \ref{ap:training objective}, it is equivalent to learn the posterior network $p_{1|t}(x^d_1|x_t)$ under the generalized Kullback–Leibler (KL) divergence:
\begin{align}\label{eq:training-posterior}
    \E_{x_1,t,x_t\sim p_{t \mid 1}(x_t\mid x_1)}\Big[\lambda_t\sum_{d=1}^\mc{D}\Big\{-\parenBig{1-\delta_{x^d_1}(x^d_t)}\log p_{1|t}^d(x^d_1|x_t)+\delta_{x^d_1}(x^d_t)-p_{1|t}^d(x^d_t|x_t)\Big\}\Big],
\end{align}
which is equivalent to Equation~37 of \citet{shaul2024flow}. If we use the masked source with the mixture path, then \Eqref{eq:training-posterior} recovers the training objective for masked diffusion models \citep{shi2024simplified,sahoo2024simple,ou2024your,nie2025large}.

\subsection{MeanFlow}\label{ap:meanflow-bg}

In MeanFlow~\citep{Geng2025MeanFF}, the authors define the average velocity field in continuous space as the displacement from time $t$ to $r$ divided by $r-t$ ($0<t<r<1$) as
$$
u\left(x_t, t, r\right) \triangleq \frac{1}{r-t} \int_t^r v\left(x_s, s\right) d s,
$$
where $v\left(x_s, s\right)$ denotes the instantaneous velocity, $t = 0$ corresponds to the noise, and $t = 1$ corresponds to the data. Then, the MeanFlow identity can be derived (see Appendix \ref{ap:meanflow_identity}):
\[
u(x_t, t, r) = v(x_t, t) - (t-r) [\partial_t u(x_t, t, r) + v(x_t, t) \cdot \nabla_{x_t} u(x_t, t, r)].
\]

Therefore, the training objective for MeanFlow can be formulated as the following fixed-point optimization problem:
\[
\mathcal{L}(\theta) = \mathbb{E}_{x_1,t,r,x_t\sim p_{t \mid 1}(x_t\mid x_1)} \left\| u_\theta(x_t, t, r) - \operatorname{sg}\big(u_{\text{tgt}}(x_t, t, r)\big) \right\|_2^2,
\]
where $u_{\text{tgt}}(x_t, t, r) = v(x_t, t) - (t - r) \left( v(x_t, t) \cdot \nabla_x u_\theta(x_t, t, r) + \partial_t u_\theta(x_t, t, r) \right)$ and $\operatorname{sg}(\cdot)$ denotes the stop-gradient operator.

\paragraph{$v$-prediction parameterization}
\citet{Geng2025ImprovedMF} identify an underlying limitation of MeanFlow, i.e., the target $u_{\text{tgt}}(x_t, t, r)$ depends on the unknown oracle $u(x_t, t, r)$,
and is therefore approximated using the network prediction $u_\theta(x_t, t, r)$.
This self-referential construction introduces bias and leads to unstable training.

To address this issue, they propose a velocity-based reparameterization derived from the MeanFlow identity:
\begin{equation}
\label{eq:v_parametrization}
V(x_t, t) = u_\theta(x_t, t, r) + (t - r)\, \left( v(x_t, t) \cdot \nabla_x u_\theta(x_t, t, r) + \partial_t u_\theta(x_t, t, r) \right).
\end{equation}

Instead of directly regressing $u$, the model is trained via a standard velocity regression objective:
\begin{equation}
\label{eq:v_parametrization_loss}
\mathbb{E}_{x_1,r,t,x_t\sim p_{t \mid 1}(x_t\mid x_1)}\big\| V(x_t, t) - v^{\text{oracle}}(x_t, t) \big\|_2^2.
\end{equation}

In practice, $v(x_t, t)$ can be obtained either via the boundary condition $v(x_t, t) = u_\theta(x_t, t, t)$
or via an auxiliary prediction head. This improved reparameterization significantly reduces variance in calculating the Jacobian-vector product, leading to more stable optimization.

\paragraph{$x$-prediction parameterization}
Beyond velocity prediction, an alternative parameterization is to directly model the clean data.
As observed by \citet{Lu2026OnestepLI}, predicting $x_1$ instead of the velocity field
better preserves the data manifold and improves training stability. Therefore, they reparameterize the MeanFlow as
\[
u_\theta(x_t, t, r) = \frac{x_\theta(x_t, t, r) - x_t}{1 - t}.
\]

Substituting this into \Eqref{eq:v_parametrization} and \Eqref{eq:v_parametrization_loss},
we obtain the training objective for MeanFlow under the $x$-prediction parameterization.
This formulation further improves empirical stability.

\section{Proposed Method}
\label{sec:dag}
\begin{table}[h]
\centering
\caption{Conceptual and mathematical comparison between MeanFlow and the proposed Discrete Average Generator.}
\renewcommand{\arraystretch}{1.35}%
\resizebox{\linewidth}{!}{
\begin{tabular}{@{}p{3.5cm}p{6.5cm}p{7.5cm}@{}}
\toprule
\textbf{Property} & \textbf{MeanFlow} & \textbf{Discrete Average Generator} \\ \midrule
\textbf{State Space} & Continuous domain ($\mathbb{R}^d$) & Discrete space ($\mathcal{S}^\mathcal{D}$) via CTMCs \\
\textbf{Instantaneous Measure} & Velocity field $v(x_t, t)$ & Transition rate generator $Q_t(x, z)$ \\
\textbf{Average Measure} & Average Velocity: \newline $u(x_t, t, r) = \frac{1}{r-t}\int_{t}^{r} v(x_s, s)ds$ & Average Generator: \newline $U_{t,r} = \frac{P_{t\to r}-I}{r-t}$
 \\
\textbf{Material Derivative} & Spatial gradient used: \newline $\frac{d}{dt}u = \partial_t u + v \cdot \nabla_x u$ & Dynkin's formula used: \newline $\frac{D}{Dt}U = \partial_t U + Q_t U$
 \\
\textbf{Fixed-Point Identity} & $u = v - (t-r)(\partial_t u + v \cdot \nabla_x u)$ & $U_{t,r} = Q_t - (t-r)(\partial_t U_{t,r} + Q_t U_{t,r})$ \\
\textbf{Stable Target \newline Parameterization} & $x$-prediction \newline (predicting clean data $x_1$) & $\tilde p$-prediction \newline (predicting a per-coordinate distribution $\tilde p_{t,r}(\cdot|x_t)$) \\ \bottomrule
\end{tabular}}
\label{tab:compa}
\end{table}

\subsection{Average Generator}

In continuous MeanFlow, acceleration is enabled by integrating the instantaneous velocity field, where a single Euler step of size $r-t$ reproduces the endpoint of the flow exactly. We define the analogue CTMC generator for discrete spaces. Let $P_{t\to r}$ denote the transition kernel of the CTMC from time $t$ to $r$, i.e., $P_{t\to r}(x,z)=p_{r|t}(z\mid x)$. We define the \emph{average generator} over $[t,r]$ as
\begin{equation}\label{eq:U-def}
 U_{t,r} \triangleq \frac{P_{t\to r}-I}{r-t},
 \qquad\text{and}\qquad U_{t,t}\triangleq Q_t,
\end{equation}
so that a single Euler step is exact,
$\delta_x^\top\big(I+(r-t)U_{t,r}\big)=p_{r|t}(\cdot|x)$ for every $x$. By Proposition~\ref{prop:U-properties} in
Appendix~\ref{ap:proof-prop1}, $U_{t,r}$ is itself a valid generator (nonnegative
off-diagonal entries, zero row sums) and satisfies $U_{t,r}\to Q_t$ as $r\to t$. We derive the following identity for the transition rate in a discrete state space, with the proof provided in Appendix~\ref{ap:proof-prop1}.

\begin{proposition}[Average Generator Identity]\label{prop:agi}
For a sufficiently smooth generator $Q_t$, the average generator satisfies:
\begin{equation}
\label{eq:meanflow_identity}
U_{t,r}
=
Q_t
-
(t-r)
\left(
\partial_t U_{t,r}
+
Q_tU_{t,r}
\right).
\end{equation}
\end{proposition}
\Eqref{eq:meanflow_identity} expresses the average generator as a self-consistent
correction of the instantaneous generator. We call $Q_tU_{t,r}$ the \emph{transport term}.
We compare the proposed framework with its continuous counterpart in Table \ref{tab:compa}. This identity forms the basis for a fixed-point training objective described in the following section.

\subsection{Training Objective of Discrete Average Generator}\label{sec:objective}

\paragraph{$U$-prediction objective.}
\Eqref{eq:meanflow_identity} motivates directly regressing $U_{t,r}^\theta$ against the
stop-gradient target constructed from the right-hand side, where $\operatorname{sg}(\cdot)$
denotes the stop-gradient operator. With
$\widehat U^{\theta}_{t,r}(x_t,z^d)\triangleq Q_t(x_t,z^d)-(t-r)\big(\partial_tU^\theta_{t,r}(x_t,z^d)+(Q_tU^\theta_{t,r})(x_t,z^d)\big)$,
the $U$-prediction objective is
\begin{equation}\label{eq:u-loss}
\mathcal{L}_U(\theta) = \mathbb{E}_{x_1,(t,r),x_t\sim p_{t \mid 1}(x_t\mid x_1)}\Bigg[\sum_{d=1}^{\mathcal{D}}
\sum_{z^d \neq x_t^d}
D_F\!\bigg(
\operatorname{sg}\!\Big(\widehat U^{\theta}_{t,r}(x_t,z^d)
\Big)
\;\bigg\|\;
U_{t,r}^\theta(x_t,z^d)
\bigg)\Bigg],
\end{equation}
with $(t,r)$, $0\le t\le r\le1$, drawn from the design of Appendix~\ref{ap:impl}. However,
analogously to the continuous setting~\citep{Geng2025MeanFF}, the target in
\Eqref{eq:u-loss} depends on $U_{t,r}^\theta$ itself, introducing a self-referential bias
that can destabilize training; moreover, unlike $U_{t,r}(x_t,z^d)$, it need not be
nonnegative, as $D_F$ requires.

\paragraph{$Q$-prediction objective.}
To resolve this issue, we follow the reparameterization strategy
of~\citet{Geng2025ImprovedMF} and express $Q_t$ directly in terms of $U_{t,r}^\theta$
via the Average Generator Identity:
\begin{equation}\label{eq:Q-reparam}
Q_{t,r}^\theta(x_t, z^d) \triangleq U_{t,r}^\theta(x_t,z^d)
+ \operatorname{sg}\!\Big(
(t-r)\Big(
\partial_t U_{t,r}^\theta(x_t,z^d)
+ \big(Q_t\,U_{t,r}^\theta\big)(x_t,z^d)
\Big)\Big).
\end{equation}
Instead of regressing $U_{t,r}^\theta$ directly, we train by regressing the
reparameterized instantaneous generator $Q_{t,r}^\theta$ against the conditional
rate $Q_t(x_t, z^d|x_1)$:
\begin{equation}\label{eq:v-loss}
\mathcal{L}_Q(\theta) = \mathbb{E}_{x_1,(t,r),x_t\sim p_{t \mid 1}(x_t\mid x_1)}\Bigg[\sum_{d=1}^{\mathcal{D}}
\sum_{z^d \neq x_t^d}
D_F\!\Big(
Q_t(x_t,z^d|x_1)
\;\Big\|\;
Q_{t,r}^\theta(x_t,z^d)
\Big)\Bigg].
\end{equation}

This eliminates the self-referential dependence in the target, which in the continuous
case yields markedly more stable optimization~\citep{Geng2025ImprovedMF}.

\paragraph{$\tilde{p}$-prediction objective.}
Analogously to the $x$-prediction parameterization in the continuous
setting~\citep{Lu2026OnestepLI}, we further reparameterize $U^\theta_{t,r}$ through a
distribution $\tilde p_{t,r}(\cdot|x)$ over $\mathcal{S}$ for each coordinate, state $x$ and
pair $(t,r)$, whose dependence on $\theta$ we leave implicit. Starting from
$Q_t^d(x,z^d)=\sum_{x_1^d}p_{1|t}^d(x_1^d|x)\,Q_t^d(x^d,z^d|x_1^d)=\lambda_t\,p_{1|t}(z^d|x)$
for $z^d\neq x^d$ and modeling $U^\theta_{t,r}$ as a generator with the same
single-coordinate structure, we set
\begin{equation}\label{eq:U-param}
U_{t,r}^\theta(x,z) = \mu_{t,r}\,\tilde p_{t,r}(z^d|x),
\quad x^{\backslash d} = z^{\backslash d},\ z^d \neq x^d,
\qquad
\mu_{t,r}\triangleq\frac{\kappa_r-\kappa_t}{(r-t)(1-\kappa_t)},
\end{equation}
where $\kappa_r$ denotes the schedule of \Eqref{eq:mixture path} evaluated at $r$. The scale $\mu_{t,r}$ makes the one-step kernel of a coordinate,
$\delta_{x^d}+(r-t)U^\theta_{t,r}(x,\cdot)$, equal to
$(1-\omega_{t,r})\delta_{x^d}+\omega_{t,r}\tilde p^{d}_{t,r}(\cdot|x)$ with
$\omega_{t,r}\triangleq(r-t)\mu_{t,r}=\frac{\kappa_r-\kappa_t}{1-\kappa_t}\in[0,1]$, the
probability that a coordinate jumps to $x_1^d$ on $[t,r]$ under the conditional path.
It satisfies
$\mu_{t,r}\to\lambda_t$ as $r\to t$, so the boundary condition $U^\theta_{t,t}=Q_t$ holds
whenever $\tilde p_{t,t}=p_{1|t}$. Directly substituting \Eqref{eq:U-param} into \Eqref{eq:Q-reparam} formally gives an objective in
 $\tilde p_{t,r}$, but its transport term $(Q_tU^\theta_{t,r})(x_t,z^d)$ poses  several obstacles. In Section \ref{sec:tractable}, we first project the transport term onto per-coordinate marginals and adopt a closed-form expression and a one-pass unbiased estimator to reduce the computational costs.

\subsection{Tractable Training Objective of Discrete Average Generator}\label{sec:tractable}

Under the coordinate-wise construction
\Eqref{eq:coordinate-wise conditional path}, $Q_t(x,y)=0$ unless $x$ and $y$ differ
in exactly one coordinate, so $Q_t$ is determined by the $\mathcal{D}\times|\mathcal{S}|$
numbers $\lambda_t\,p_{1|t}(z^d|x)$ per state, which we obtain from a frozen pretrained
denoiser, and $U^\theta_{t,r}$ has the same structure by \Eqref{eq:U-param}. We therefore
model $\tilde p_{t,r}$ with a network $H^\theta$ that maps $(x,t,r)$ to $\mathcal{D}$
distributions over $\mathcal{S}$, $\tilde p_{t,r}(z^d|x)=[H^\theta(x,t,r)]_{d,z^d}$, so that a
single forward pass gives every entry of $U^\theta_{t,r}$ at the current state.

\paragraph{Per-coordinate projection.}
Let $x^{d\to s}$ denote the state $x$ with
its $d$-th coordinate replaced by $s$. For each coordinate $d$, the
\emph{per-coordinate projection} $\Pi^d$ is the linear map that marginalizes the target
state of a matrix onto its $d$-th coordinate: it takes a matrix $M$ indexed by pairs of
states $(x,y)\in\mathcal{S}^{\mathcal{D}}\times\mathcal{S}^{\mathcal{D}}$ and returns the matrix $\Pi^d[M]$ indexed by
$(x,z^d)\in\mathcal{S}^{\mathcal{D}}\times\mathcal{S}$,
\begin{equation}\label{eq:projection}
\Pi^{d}[M](x,z^{d})\;\triangleq\!\!\sum_{y\in\mathcal{S}^{\mathcal{D}}:\;y^{d}=z^{d}}\!\!M(x,y),
\qquad z^d\neq x^d,
\end{equation}
i.e.\ $\Pi^d[M]=ME_d$ with $E_d(y,z^d)\triangleq\mathbbm{1}[y^d=z^d]$, which sums the row
$M(x,\cdot)$ over all target states whose $d$-th coordinate equals $z^d$, whatever the
other coordinates are. For a transition kernel, $\Pi^d[M](x,\cdot)$ is the law of the
$d$-th coordinate after the transition. For a generator, $\Pi^d[M](x,z^d)$ is the total
rate at which coordinate $d$ moves to $z^d$. For a generator that changes only one coordinate at a time, the
d-th coordinate transition is given by a single matrix entry, $\Pi^{d}[Q_t](x,z^d)=Q_t(x,x^{d\to z^d})=\lambda_t\,p_{1|t}(z^d|x)$ and
$\Pi^{d}[U^\theta_{t,r}](x,z^d)=\mu_{t,r}\tilde p_{t,r}(z^d|x)$. For the exact average generator, however, the
d-th coordinate transition is given  by the
d-th marginal of the jump kernel:
$\Pi^{d}[U_{t,r}](x,z^d)=p^{d}_{r|t}(z^d|x)/(r-t)$, which is exactly the quantity required by a coordinate-factorized sampler. In contrast, the entry $U_{t,r}(x,x^{d\to z^d})$ is proportional to the probability that coordinate $d$ jumps to $z^d$ while all other coordinates remain unchanged, with proportionality factor $1/(r-t)$. Applying $\Pi^d$ to
\Eqref{eq:meanflow_identity} gives, for every $d$ and $z^d\neq x^d$, the projected identity
underlying Section~\ref{sec:objective},
\begin{align}\label{eq:projected-identity}
\Pi^{d}[U_{t,r}](x,z^d)&=\Pi^{d}[Q_t](x,z^d)-(t-r)\Big(\partial_t\Pi^{d}[U_{t,r}](x,z^d)+\Pi^{d}[Q_tU_{t,r}](x,z^d)\Big)\\ \notag
&=\Pi^{d}[Q_t](x,z^d)-(t-r)\Big(\partial_t\Pi^{d}[U_{t,r}](x,z^d)+\sum_{x'}Q_t(x,x')\,\Pi^{d}[U_{t,r}](x',z^d)\Big).
\end{align}
Here we use that $\Pi^d$ is
linear and commutes with $\partial_t$, and that $\Pi^{d}[Q_tM](x,z^d)=\sum_{x'}Q_t(x,x')\,\Pi^{d}[M](x',z^d)$.
The right-hand side involves $U_{t,r}$ only through $\Pi^{d}[U_{t,r}]$ at $x$ and at its
single-coordinate neighbors. The projection therefore gives the transport term a
well-defined per-coordinate entry and absorbs the two-coordinate moves by summing over them.
We call $\Pi^{d}[Q_tU_{t,r}]$ the \emph{projected transport term}.

\paragraph{Tractable projected transport term.}
The projected identity \Eqref{eq:projected-identity} is a \emph{closed} equation for the
per-coordinate marginals: the
unique solution of \Eqref{eq:projected-identity} is the exact marginal $\Pi^{d}[U_{t,r}]$, which is proved in Proposition~\ref{prop:fixed-point}. A sufficiently expressive model trained with
\Eqref{eq:v-loss} therefore learns the per-coordinate marginals of the finite-step kernel,
using only the instantaneous rate and quantities at the current state, without simulating
the kernel over $[t,r]$. It remains to evaluate the projected transport term for the
parameterization \Eqref{eq:U-param}. We show in Proposition~\ref{prop:tractable-QU} that
\begin{equation}\label{eq:commutator}
\Pi^{d}\big[Q_t\,U_{t,r}^\theta\big](x_t,z^{d})
=\lambda_t\,\mu_{t,r}\Big[\,C^{d}_{\theta}(x_t,z^{d})-p_{1|t}(z^{d}\mid x_t)\Big],
\end{equation}
where
\begin{equation}\label{eq:C-term}
  C^{d}_{\theta}(x_t,z^{d})
  \;\triangleq\;
  \sum_{e=1}^{\mathcal{D}}
  \Big(
  \mathbb{E}_{s\sim p^{e}_{1\mid t}(\cdot\mid x_t)}
  \big[\tilde p_{t,r}\big(z^{d}\mid x_t^{e\to s}\big)\big]
  -\tilde p_{t,r}\big(z^{d}\mid x_t\big)
  \Big)
\end{equation}
measures how the prediction at coordinate $d$ moves when a single coordinate of $x_t$ jumps
according to $Q_t$. Each summand requires evaluating the network at a state that differs from
$x_t$ in a single coordinate, producing an output of size $\mathcal{D}\times|\mathcal{S}|$. Thus, evaluating the sum exactly requires one forward pass for each such state, motivating the estimator introduced below.
Substituting \Eqref{eq:commutator} and
$\partial_t\mu_{t,r}=\frac{\mu_{t,r}-\lambda_t}{r-t}+\lambda_t\mu_{t,r}$ into
\Eqref{eq:Q-reparam}, the purely local terms combine and the reparameterized instantaneous
rate becomes
\begin{equation}\label{eq:At-tractable}
Q_{t,r}^\theta(x_t,z^{d})
= \lambda_t\,\tilde p_{t,r}(z^{d}\mid x_t)
+ \operatorname{sg}\Big\{(t-r)\,\mu_{t,r}\Big[
      \partial_t\tilde p_{t,r}(z^{d}\mid x_t)
      +\lambda_t\,\Gamma^{d}_{\theta}(x_t,z^{d})\Big]\Big\},
\end{equation}
with
\begin{equation}\label{eq:Gamma}
\Gamma^{d}_{\theta}(x_t,z^{d})\;\triangleq\;C^{d}_{\theta}(x_t,z^{d})
+\tilde p_{t,r}(z^{d}\mid x_t)-p_{1|t}(z^{d}\mid x_t).
\end{equation}
Writing \Eqref{eq:At-tractable} with $\lambda_t\tilde p_{t,r}$ outside the
stop-gradient rescales the gradient at each $(t,r)$ by $\lambda_t/\mu_{t,r}$ relative to
\Eqref{eq:Q-reparam}, which leaves the fixed point unchanged. For the linear schedule
$\kappa_t=t$ the two coincide, since $\mu_{t,r}=\lambda_t$. The correction
therefore consists of exactly two terms: the partial time derivative of the network at
fixed $r$ and $\lambda_t\Gamma^{d}_{\theta}$, which contains the projected transport term: by \Eqref{eq:commutator} and \Eqref{eq:Gamma}, $\lambda_t\mu_{t,r}\Gamma^{d}_{\theta}=\Pi^{d}[Q_tU^\theta_{t,r}](x_t,z^d)+\lambda_t\mu_{t,r}\tilde p_{t,r}(z^d\mid x_t)$, where the last term comes from $\partial_t\mu_{t,r}$. At $r=t$ both vanish and \Eqref{eq:At-tractable} reduces to
the instantaneous rate $Q^\theta_{t,t}(x_t,z^d)=\lambda_t\,\tilde p_{t,t}(z^d\mid x_t)$, which equals $Q_t(x_t,z^d)$ when $\tilde p_{t,t}=p_{1|t}$, so the objective is a strict
generalization of \Eqref{eq:training-posterior}.

$C^{d}_{\theta}$ is estimated without bias by a single forward pass shared across
all coordinates: draw $e\sim\mathrm{Unif}(\{1,\dots,\mathcal{D}\})$ and
$s\sim p^{e}_{1|t}(\cdot|x_t)$, evaluate $\tilde p_{t,r}(\cdot\mid x_t^{e\to s})$ once, and
set
\begin{equation}\label{eq:C-estimator}
  \widehat C^{d}_{\theta}(x_t,z^{d})
  =\mathcal{D}\,\Big(\tilde p_{t,r}\big(z^{d}\mid x_t^{e\to s}\big)
  -\tilde p_{t,r}\big(z^{d}\mid x_t\big)\Big).
\end{equation}
Every summand of $C^{d}_{\theta}$, as well as $\tilde p_{t,r}-p_{1|t}$ and
$\partial_t\tilde p_{t,r}$ sums to zero
over $z^d$. Therefore, the shifted probability
\begin{equation}\label{eq:data-posterior}
    \hat q^{d}_{\theta}(z^d|x_t)\triangleq\tilde p_{t,r}(z^d|x_t)+\operatorname{sg}\{\frac{(t-r)\mu_{t,r}}{\lambda_t}[\partial_t\tilde p_{t,r}(z^d|x_t)+\lambda_t\Gamma^{d}_{\theta}(x_t,z^d)]\},
\end{equation}
defined for all $z^d\in\mathcal{S}$, satisfies $\sum_{z^d\in\mathcal{S}} \hat q^{d}_{\theta}(z^d\mid x_t)=1$. Moreover, for $z^d\neq x_t^d$, $\hat q^{d}_{\theta}(z^d|x_t)=Q^\theta_{t,r}(x_t,z^d)/\lambda_t$. Expanding $D_F$ in
\Eqref{eq:v-loss} (Appendix~\ref{ap:dag-objective-proof}), its linear part becomes
$\lambda_t(1-\hat q^{d}_{\theta}(x_t^d|x_t))$, and the objective takes the form of
\Eqref{eq:training-posterior} with the posterior replaced by the shifted probability
$\hat q^{d}_{\theta}$,
\begin{equation}\label{eq:loss-general}
  \mathcal{L}_{\tilde{p}}(\theta)
  =\mathbb{E}_{x_1,(t,r),x_t}\Bigg[\lambda_t\sum_{d=1}^{\mathcal{D}}\Big\{
  -\big(1-\delta_{x_1^d}(x_t^d)\big)\log\hat q^{d}_{\theta}(x_1^{d}|x_t)
  -\hat q^{d}_{\theta}(x_t^{d}|x_t)\Big\}\Bigg]+\mathrm{const}.
\end{equation}

\begin{algorithm}[t]
\caption{Training of Discrete Average Generator}
\label{alg:dag-training}
\begin{algorithmic}[1]
\Require Network $H^\theta$, schedule $\kappa_t$, data distribution $p_1$,
  frozen denoiser $H$, boundary fraction $1-\rho$
\Repeat
    \State Sample $x_1 \sim p_1$, $(t,r)$ with $t\le r$ from the logit-normal; set $r \gets t$ with probability $1-\rho$
    \State Sample $x_t \sim p_{t|1}(\cdot \mid x_1)$ \Comment{mixture path}
    \State Compute $\tilde p_{t,r}(\cdot \mid x_t) \gets H^\theta(x_t,t,r)$, $\partial_t \tilde p_{t,r}(z^d \mid x_t)$ at \emph{fixed} $r$
    \Comment{$\mathcal{D}\times|\mathcal{S}|$}
    \State Compute $p_{1|t}(\cdot \mid x_t) \gets H(x_t,t)$
    \Comment{frozen, stop-gradient}
    \State Sample $e \sim \mathrm{Unif}(\{1,\dots,\mathcal{D}\})$, $s \sim p_{1|t}^e(\cdot\mid x_t)$; evaluate $\tilde p_{t,r}(\cdot \mid x_t^{e\to s})$
    \State Compute $\widehat{C}^d_{\theta}$ via \Eqref{eq:C-estimator} and $\Gamma^d_\theta$ via \Eqref{eq:Gamma}
    \State Compute the shifted probability $\hat q^{d}_{\theta}$ with stop-gradient via \Eqref{eq:data-posterior}
    \State Compute $\mathcal{L}_{\tilde{p}}(\theta)$ via \Eqref{eq:loss-general}
    \State Update $\theta \gets \theta - \eta \nabla_\theta \mathcal{L}_{\tilde{p}}(\theta)$
\Until{converged}
\end{algorithmic}
\end{algorithm}
Algorithm~\ref{alg:dag-training} summarizes training algorithm. The instantaneous posterior
$p_{1|t}$ is supplied by a frozen pretrained denoiser, which is also the model the network
is initialized from; the boundary condition $U^\theta_{t,t}=Q_t$ gives the same quantity when
no pretrained model is available, since a fraction $1-\rho$ of the training pairs are drawn
with $r=t$.

\paragraph{Sampling.}
Sampling takes one Euler step of the learned
per-coordinate average generator $\Pi^{d}[U^\theta_{t,r}]$ independently for every coordinate per interval. On a
grid $0=\tau_0<\dots<\tau_K=1$, each step queries the network once at
$(x_t,t,r)=(x_{\tau_k},\tau_k,\tau_{k+1})$ and, independently for every coordinate,
replaces it by a draw from $\tilde p^{d}_{t,r}(\cdot\mid x_t)$ with probability
$\omega_{t,r}=(r-t)\mu_{t,r}=\frac{\kappa_r-\kappa_t}{1-\kappa_t}$. The detailed sampling procedure is given in Algorithm~\ref{alg:dag-sampling}.

\begin{figure}[t]
\centering
\begin{minipage}[t]{0.47\linewidth}
  \centering
  \includegraphics[width=0.8\linewidth]{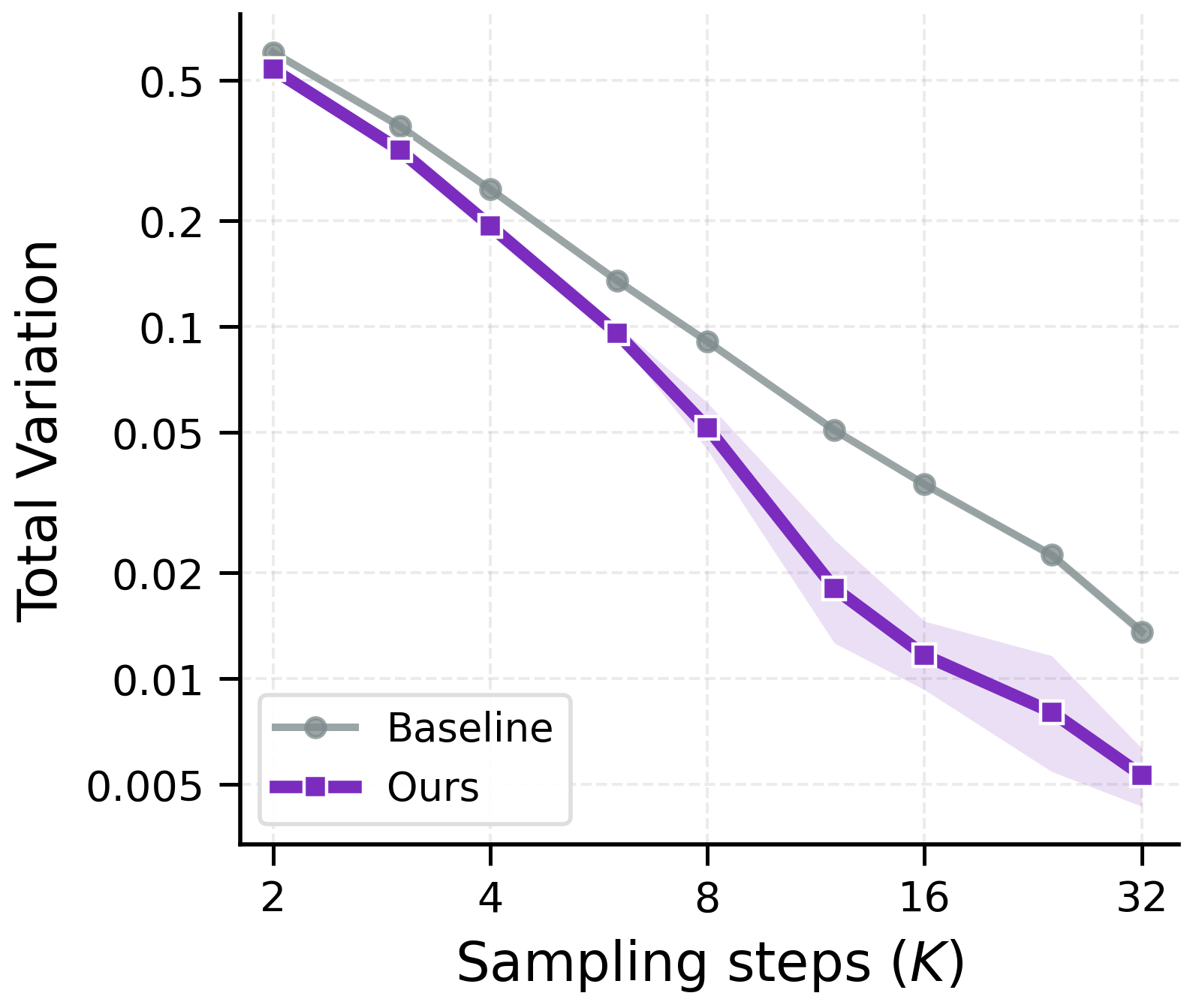}
  \caption{Exact total variation between the target and the law of the $K$-step sampler on
  the Potts simulation. Values in Table~\ref{tab:sim-full}.}
  \label{fig:simulations}
\end{minipage}\hfill
\begin{minipage}[t]{0.47\linewidth}
  \centering
  \includegraphics[width=0.8\linewidth]{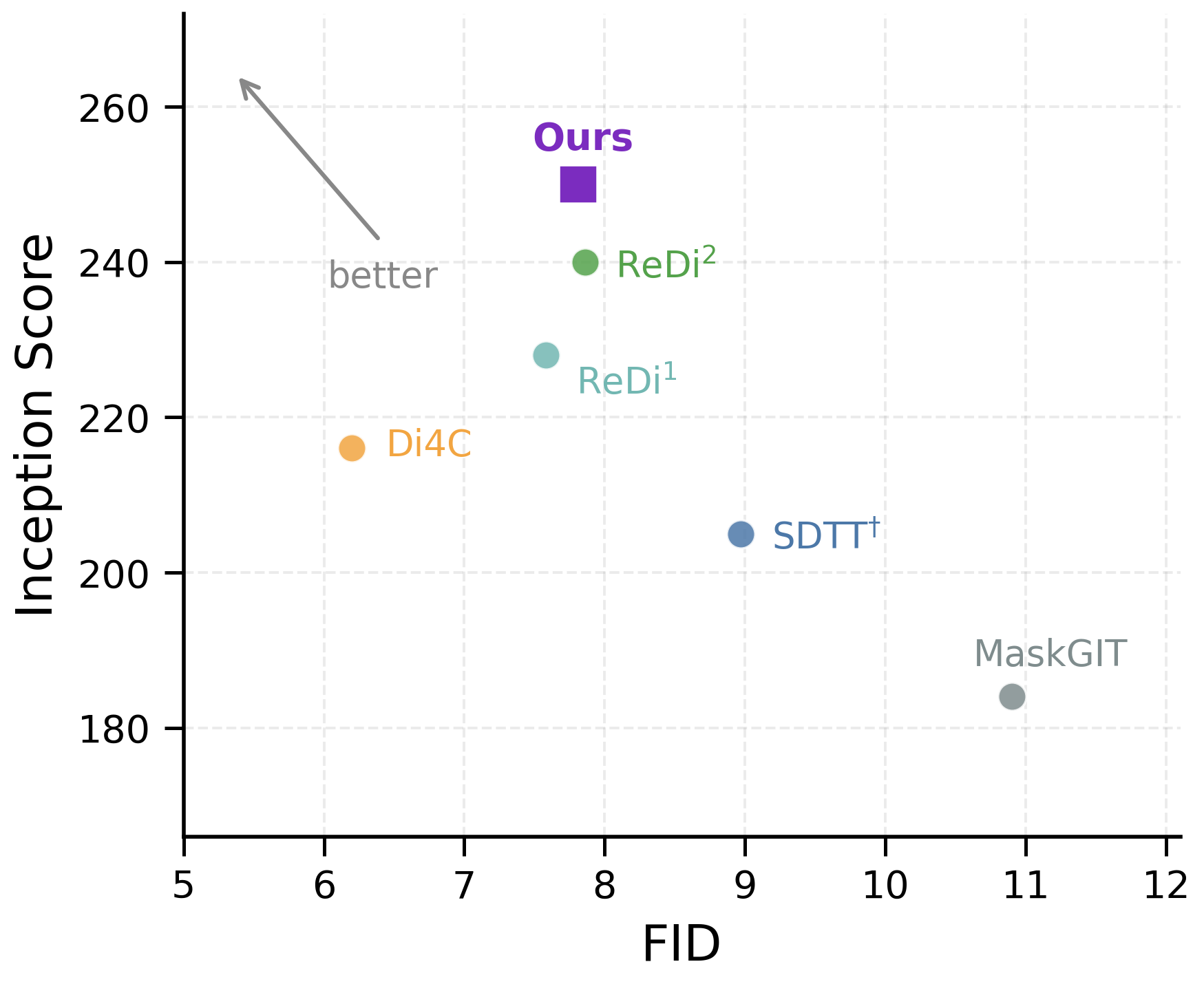}
  \caption{Experimental results on the ImageNet dataset for image generation. Values in Table~\ref{tab:main_results}.}
  \label{fig:imagenet}
\end{minipage}
\end{figure}
\section{Experimental Results}\label{sec:exper}
\subsection{Experimental Setup}

\paragraph{Datasets.}
We evaluate our method on simulated data and two standard benchmarks: OpenWebText~\citep{Gokaslan2019OpenWeb} for text generation and ImageNet~\citep{deng2009imagenet} for image generation. For the simulations, we use a Potts target on $\mathcal{S}^{\mathcal{D}}$
with $|\mathcal{S}|=\mathcal{D}=4$, a state space small enough to be enumerated exactly. This allows us to evaluate the sampler against the exact ground truth in terms of total variation distance. Details of the data generation procedure are provided in Appendix~\ref{ap:sim-exact}. For text generation, we use sequences of length 1024, tokenized with the GPT-2 tokenizer~\citep{Radford2019LanguageMA}, which has a vocabulary size of 50,257. For image generation, following prior work~\citep{chang2022maskgit,hayakawa2025distillation}, we tokenize $256 \times 256$ images into $16 \times 16$ discrete latent codes using a pretrained VQGAN~\citep{Esser2020TamingTF}, resulting in sequences of length 256 with a vocabulary size of 1,024.

\paragraph{Baselines.}

We compare with SDTT~\citep{deschenaux2024auto}
and Di4C~\citep{hayakawa2025distillation}, which distill a
multi-step teacher and, for Di4C, model the correlations between dimensions, and with
ReDi~\citep{Yoo2025ReDiRD}, which iteratively rectifies the source--data coupling
(ReDi$^{k}$: $k$ rounds; Di4C+ReDi$^{1}$ combines both). Our method does not distill a
student from a teacher trajectory.

\paragraph{Implementation.}
On OpenWebText, we
train the model initialized from Di4C+ReDi$^{1}$ with \Eqref{eq:v-loss} under the parameterization \Eqref{eq:duo-U} and the estimator
of \Eqref{eq:transport-general} (Appendix~\ref{ap:duo}) for 1k
iterations. On ImageNet, we train the model initialized from ReDi$^{2}$ with \Eqref{eq:loss-masked} and the single-draw
estimator \Eqref{eq:C-estimator-masked} (Appendix~\ref{ap:masked}) for 3
epochs.
Further details can be found in Appendix~\ref{ap:impl}.

\paragraph{Metrics.}
For simulated data, we report the total variation distance between the generated distribution and the ground-truth distribution. For text generation, we measure generative perplexity using LLaMA 3.1-8B~\citep{grattafiori2024llama} and report entropy to assess diversity and degeneration. For image generation, we generate 50k samples and report Fréchet Inception Distance (FID)~\citep{Heusel2017GANs} and Inception Score (IS)~\citep{Tim2016improved}.

\subsection{Experimental Results}
\label{sec:sim}

\paragraph{Simulations}
Figure~\ref{fig:simulations} compares a network trained
with the baseline objective \Eqref{eq:training-posterior} with one trained with \Eqref{eq:loss-general}. The correction
reduces the total variation for every number of sampling steps $K$, by $21\%$ at $K=4$,
$43\%$ at $K=8$ and $67\%$ at $K=16$, with the same ordering for all three seeds.

\paragraph{Text Generation}
Figure~\ref{fig:openweb} reports generative perplexity
and entropy on OpenWebText. From $K=8$ on, our method improves on the strongest baseline,
Di4C+ReDi$^{1}$: $27.59$ vs.\ $30.92$ at $K=8$, $22.92$ vs.\ $24.81$ at $K=16$, $21.35$ vs.\ $21.88$ at
$K=32$, and $16.97$ vs.\ $20.32$ at $K=64$. Most importantly, our method achieves a lower generative perplexity with $K=64$ than Di4C+ReDi$^{1}$ with $K=1024$, demonstrating a $16\times$ acceleration. Qualitative inspection of unconditional samples generated at $K=64$ (Appendix~\ref{app:samples_k64}) further suggests that our method produces fluent, topically coherent text at this reduced step budget, without the repetitive degeneration that can mask a low generative perplexity.

\begin{figure}
    \centering
    \includegraphics[width=0.6\linewidth]{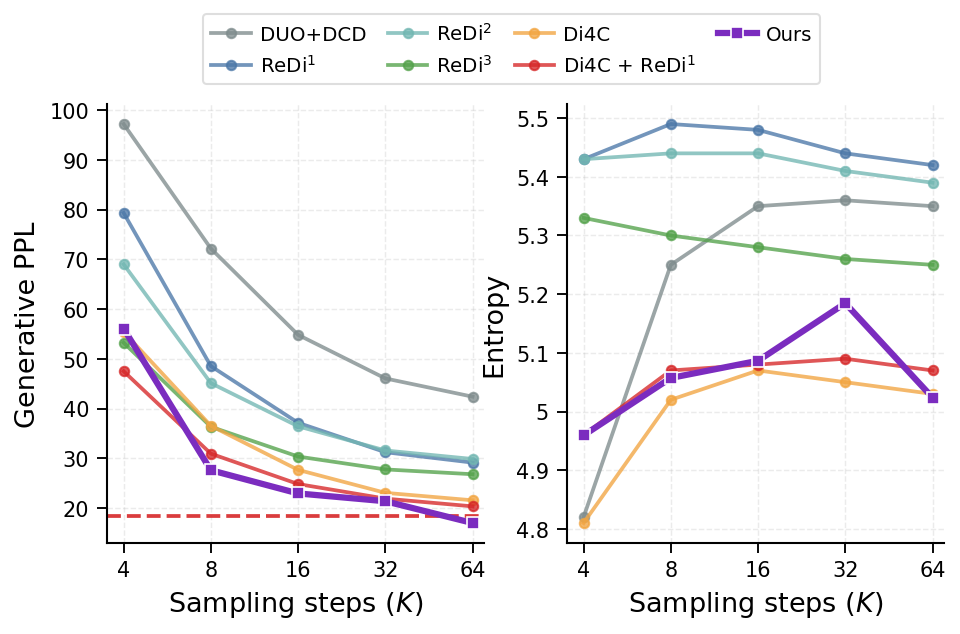}%
    \caption{Generative perplexity and entropy on
    OpenWebText. The dashed horizontal line indicates the PPL of Di4C + ReDi$^1$ at $K=1024$ (18.44).
    Values in Table~\ref{tab:openwebtext}.}
    \label{fig:openweb}
\end{figure}

\paragraph{Image Generation}

Figure~\ref{fig:imagenet} compares the methods on
ImageNet $256\times256$ with four network evaluations. Our method attains the highest Inception Score, $250$, $4\%$ above the next best
 (ReDi$^{2}$, $240$), at a comparable FID ($7.81$ vs.\ $7.86$).
The ablation of classifier-free guidance can be found in Appendix~\ref{ap:imagenet-cfg}
 and qualitative comparisons are provided in
Appendix~\ref{ap:imagenet-qual}.

\section{Related Work}
In continuous spaces, sampling is accelerated by higher-order
solvers~\citep{Song2020DenoisingDI,Lu2022DPMSolverAF}, caching~\citep{ma2023deepcache},
distillation~\citep{salimans2022progressive,zhou2024score}, consistency
models~\citep{Song2023ConsistencyM}, rectified flow~\citep{Liu2022FlowSA}, and methods that
learn the solution operator over an interval, including shortcut models~\citep{frans2025one},
MeanFlow~\citep{Geng2025MeanFF,Geng2025ImprovedMF,Lu2026OnestepLI}. For discrete diffusion,
SDTT~\citep{deschenaux2024auto} and Di4C~\citep{hayakawa2025distillation} distill
multi-step samplers, ReDi~\citep{Yoo2025ReDiRD} rectifies the coupling,
Di[M]O~\citep{Zhu2025DiMODM} distills a one-step generator by distribution matching, and
Fast-dLLM~\citep{Wu2026FastdLLMTA} and Jacobi Forcing~\citep{Hu2026FastAA} accelerate
inference. Closest to our work, Categorical Flow Maps~\citep{Roos2026CategoricalFM} and
Discrete Flow Maps~\citep{Potaptchik2026DiscreteFM} learn flow maps of a \emph{continuous}
interpolant on the probability simplex. We instead stay on the discrete state space and
learn the average generator of a CTMC, so that the transport term is left multiplication
by the generator rather than a Jacobian--vector product along a drift, and the
construction applies to any CTMC with coordinate-wise rates.
Appendix~\ref{ap:related} provides a more detailed discussion of related work.

\section{Conclusion}
We have introduced the Discrete Average Generator, a principled extension of MeanFlow
to discrete state spaces governed by CTMCs. By defining the average
generator as the normalized increment of the transition kernel, we derived the Discrete
Average Generator Identity and the corresponding training objective. Projected onto per-coordinate marginals,
the identity becomes a closed equation whose transport term can be estimated without bias with one additional forward pass. Our method reduces the total variation of the $K$-step sampler
by up to $67\%$ in simulations, attains the lowest generative perplexity of all compared
methods on OpenWebText for every budget of $8$ to $64$ steps with 16 $\times$ acceleration, and on ImageNet with
the highest Inception Score.

\bibliography{references}

\appendix
\section{Detailed Discussion of Related Work}\label{ap:related}
\paragraph{Acceleration for Continuous Diffusion Models and Flow Matching.}
DDIM~\citep{Song2020DenoisingDI} and DPM-Solver~\citep{Lu2022DPMSolverAF} leverage
deterministic or higher-order ODE solvers to speed up sampling of diffusion models
and flow matching. Caching-based methods, such as DeepCache~\citep{ma2023deepcache},
reuse intermediate computations across timesteps to further improve efficiency.
Another line of approaches focuses on distilling the original multi-step generation
process into fewer steps. Progressive distillation~\citep{salimans2022progressive}
iteratively compresses the sampling trajectory, while score identity
distillation~\citep{zhou2024score} and consistency models~\citep{Song2023ConsistencyM}
enforce consistency across different noise levels to enable one- or few-step generation.
Rectified flow~\citep{Liu2022FlowSA} provides another perspective by reformulating the
generative process as a straightened flow, which naturally reduces the complexity of
sampling. Shortcut models~\citep{frans2025one} and MeanFlow~\citep{Geng2025MeanFF}
introduce tailored modifications to the flow formulation to achieve faster generation.
In particular, MeanFlow defines an average velocity field over time intervals and
derives a self-consistent fixed-point identity, which is the direct inspiration for
our work. Improved MeanFlow~\citep{Geng2025ImprovedMF} further addresses training
instability through a velocity-based reparameterization, and pixel MeanFlow~\citep{Lu2026OnestepLI}
proposes an $x$-prediction parameterization that improves empirical stability by
directly supervising the clean data prediction.

\paragraph{Acceleration of Diffusion Language Models.}
Extending acceleration techniques to discrete diffusion models and diffusion language
models has attracted increasing attention. SDTT~\citep{deschenaux2024auto} and
Di4C~\citep{hayakawa2025distillation} accelerate discrete diffusion by distilling
multi-step trajectories into fewer steps, where Di4C additionally exploits
dimensional correlations to improve distillation quality.
ReDi~\citep{Yoo2025ReDiRD} iteratively rectifies the coupling between source
and target distributions to reduce the factorization error of discrete flow models,
enabling efficient few-step generation. Di[M]O~\citep{Zhu2025DiMODM} proposes to distill masked diffusion models
into a one-step generator via token-level distribution matching in an on-policy
framework, and introduces a token initialization strategy to inject stochasticity
into the otherwise deterministic initial distribution.
On the inference side, Fast-dLLM~\citep{Wu2026FastdLLMTA} accelerates diffusion LLM
inference without additional training by introducing a block-wise approximate KV
cache mechanism for bidirectional diffusion models and a confidence-aware parallel
decoding strategy. Jacobi Forcing~\citep{Hu2026FastAA} takes a complementary approach by
converting autoregressive LLMs into causal parallel decoders through progressive
distillation on the model's own Jacobi decoding trajectories, preserving the causal
backbone and KV cache compatibility while achieving significant wall-clock speedup.
In contrast to these approaches, our method directly learns an average generator for
CTMCs without multi-step distillation, and addresses the discrete analogue of the
MeanFlow fixed-point identity at the level of the transition rate operator.

\section{Theoretical Details for Section \ref{sec:prel}}

\subsection{Proof of Training Objective for Discrete Flow Matching} \label{ap:training objective}

\begin{proof}
We show how the Bregman divergence-based training objective
\[
\mathbb{E}\Big[\sum_{z\neq x_t} D_F\big(Q_t(x_t,z|x_1) \,\|\, Q_t(x_t,z)\big)\Big],
\quad F(x) = x \log x,
\]
reduces to a cross-entropy loss with respect to the posterior $p_{1|t}$ when using the coordinate-wise mixture path.

\textbf{Step 0: Off-diagonal restriction.}
The generator $Q_t(x,z|x_1)$ contains a negative diagonal term $-\delta_x(z)$ and is therefore a signed measure. Since the objective only sums over $z \neq x$, we restrict both conditional and marginal generators to their off-diagonal parts:
\[
Q_t^{\mathrm{off}}(x,z|x_1) := Q_t(x,z|x_1)\mathbf{1}_{z \neq x}.
\]
Under this restriction, all quantities are non-negative and the Bregman divergence is well-defined.

\textbf{Step 1: Bregman divergence.}
For $F(x)=x\log x$, the Bregman divergence is
\[
D_F(a\|b) = a \log \frac{a}{b} - a + b.
\]

\textbf{Step 2: Conditional and oracle rates.}
From the coordinate-wise mixture path, for a single coordinate $d$,
\[
Q_\text{cond}(z^d)
= Q_t^d(x_t^d,z^d|x_1^d)\mathbf{1}_{z^d \neq x_t^d}
= \lambda_t
\mathbf{1}_{\{x_1^d\neq x_t^d\}}\, \delta_{x_1^d}(z^d),
\]
where we used that the original generator contains an additional diagonal term $-\delta_{x_t^d}(z^d)$, which vanishes after restricting to $z^d \neq x_t^d$.
The corresponding oracle rate is
\[
Q_\text{oracle}(z^d) = Q_t^d(x_t^d,z^d)
= \lambda_t\, p_{1|t}^d(z^d|x_t),
\quad z^d \neq x_t^d.
\]

\textbf{Step 3: Support of the conditional rate.}
The conditional rate has support only at $z^d = x_1^d$ when $x_1^d\neq x_t^d$, and is identically zero otherwise. Hence the analysis splits into two cases.

\textbf{Step 4: KL term at the target jump ($x_1^d\neq x_t^d$).}
For $z^d = x_1^d$:
\[
Q_\text{cond}(x_1^d) \log \frac{Q_\text{cond}(x_1^d)}{Q_\text{oracle}(x_1^d)}
= \lambda_t \log \frac{\lambda_t}{\lambda_t \, p_{1|t}^d(x_1^d|x_t)}
= \lambda_t \, \big(-\log p_{1|t}^d(x_1^d|x_t)\big).
\]

\textbf{Step 5: Linear terms.}
The remaining terms in $D_F$ give
\[
-Q_\text{cond}(x_1^d) + \sum_{z^d \neq x_t^d} Q_\text{oracle}(z^d)
= - \lambda_t
+ \lambda_t \big(1 - p_{1|t}^d(x_t^d|x_t)\big).
\]

\textbf{Step 6: Case $x_1^d=x_t^d$.}
In this case, $Q_\text{cond}(z^d)=0$ for all $z^d \neq x_t^d$. Hence
\[
D_F(0\|Q_\text{oracle}(z^d)) = Q_\text{oracle}(z^d),
\]
and summing over $z^d \neq x_t^d$ yields
\[
\lambda_t (1 - p_{1|t}^d(x_t^d|x_t)).
\]

\textbf{Step 7: Combine terms.}
Combining both cases using the indicator $\delta_{x_1^d}(x_t^d)$, the Bregman divergence reduces to
\[
\lambda_t \Big\{
- (1-\delta_{x_1^d}(x_t^d)) \log p_{1|t}^d(x_1^d|x_t)
+ \delta_{x_1^d}(x_t^d) - p_{1|t}^d(x_t^d|x_t)
\Big\}.
\]
Summing over all coordinates $d=1,\dots,\mathcal{D}$, we obtain exactly the training objective
\[
\mathbb{E}\Big[\lambda_t \sum_{d=1}^{\mathcal{D}} \Big\{
- (1-\delta_{x_1^d}(x_t^d)) \log p_{1|t}^d(x_1^d|x_t)
+ \delta_{x_1^d}(x_t^d) - p_{1|t}^d(x_t^d|x_t)
\Big\}\Big].
\]
This completes the proof.
\end{proof}

\subsection{Proof of Proposition 3.1 of \citet{campbell2024generative}}
\label{appendix:prop}
We include the proof of Proposition 3.1 of \citet{campbell2024generative} here for completeness.
Let $Q_t(x,z|x_1)$ denote the conditional transition rate from state $x$ to $z$ that generates
the conditional probability path $p_{t|1}(\cdot|x_1)$, i.e., it satisfies the Kolmogorov forward equation
\[
\partial_t p_{t|1}(x|x_1)
= \sum_{z} Q_t(z,x|x_1)\, p_{t|1}(z|x_1).
\]

Taking expectation over $x_1 \sim p_1$, we obtain
\begin{align*}
\partial_t p_t(x)
&= \mathbb{E}_{x_1\sim p_{1}(x_1)}\big[\partial_t p_{t|1}(x |x_1)\big]
\quad{\small\text{(definition of $p_t$)}} \\
&= \mathbb{E}_{x_1\sim p_{1}(x_1)}\big[\sum_z Q_t(z, x | x_1) \, p_{t|1}(z | x_1)\big] \\
&= \sum_z \sum_{x_1} p_{1}(x_1)\,p_{t|1}(z | x_1)\,Q_t(z,x| x_1) \\
&= \sum_z \sum_{x_1} p_t(z)\,p_{1|t}(x_1| z)\,Q_t(z,x| x_1)
\quad{\small\text{(Bayes' rule)}} \\
&= \sum_z \mathbb{E}_{x_1\sim p_{1|t}(\cdot | z)}\big[Q_t(z,x| x_1)\big]\,p_t(z).
\end{align*}

Comparing with the Kolmogorov forward equation
\[
\partial_t p_t(x)
= \sum_z Q_t(z,x)\,p_t(z),
\]
we identify the marginal transition rate as
\[
Q_t(z,x)
= \mathbb{E}_{x_1\sim p_{1|t}(\cdot | z)}\big[Q_t(z,x|x_1)\big].
\]

Equivalently, reparameterizing the indices to express the transition rate
from $x$ to $z$, we obtain
\[
Q_t(x,z)
= \mathbb{E}_{x_1\sim p_{1|t}(\cdot | x)}\big[Q_t(x,z|x_1)\big].
\]

This shows that $Q_t(x,z)$ generates the marginal probability path $p_t(x)$.

\subsection{Proof of the MeanFlow Identity} \label{ap:meanflow_identity}

The MeanFlow identity was originally proved by \citet{Geng2025MeanFF}.
We reproduce the proof here for completeness.

Let $x_t$ follow the flow defined by $\frac{d}{dt} x_t = v(x_t, t)$.
By definition, the MeanFlow velocity is
\[
u(x_t, t, r) = \frac{1}{r - t} \int_t^r v(x_s, s)\, ds.
\]
Multiplying both sides by $(r - t)$ gives
\[
(r - t)\, u(x_t, t, r) = \int_t^r v(x_s, s)\, ds.
\]
Taking the total derivative with respect to $t$ and applying the
fundamental theorem of calculus yields
\[
\frac{d}{dt} \big[(t - r)\, u(x_t, t, r)\big] = v(x_t, t).
\]
Expanding the left-hand side using the product rule, we obtain
\[
(t - r)\, \frac{d}{dt} u(x_t, t, r) + u(x_t, t, r) = v(x_t, t).
\]
Rearranging terms leads to the MeanFlow identity
\[
u(x_t, t, r) = v(x_t, t) - (t - r)\, \frac{d}{dt} u(x_t, t, r).
\]

Here, the total derivative is given by
\[
\frac{d}{dt} u(x_t, t, r)
= \partial_t u(x_t, t, r)
+ v(x_t, t) \cdot \nabla_{x_t} u(x_t, t, r).
\]

\section{Theoretical Details}\label{ap:theoretical}

\subsection{Notation and the material derivative of a CTMC}
\label{ap:material}

Throughout this appendix, $x^{d\to s}$ denotes the state obtained from $x$ by replacing
its $d$-th coordinate with $s$, matrices are indexed by states and act on the right on
row-vectors of probabilities, and $P_{t\to r}$ denotes the propagator of the CTMC,
$P_{t\to r}(x,y)=p_{r|t}(y|x)$, which is the unique solution of
\begin{equation}
\label{eq:kolmogorov-fwd-bwd}
  \underbrace{\partial_r P_{t\to r}=P_{t\to r}\,Q_r}_{\text{forward}},
  \qquad
  \underbrace{\partial_t P_{t\to r}=-\,Q_t\,P_{t\to r}}_{\text{backward}},
  \qquad P_{t\to t}=I .
\end{equation}
For each coordinate $e$ define the row-stochastic matrix $P^{e}_{t}$ that resamples
coordinate $e$ from the posterior,
$P^{e}_{t}(x,x^{e\to s})\triangleq p^{e}_{1|t}(s|x)$ for all $s\in\mathcal{S}$ (the
diagonal $s=x^e$ included) and $P^{e}_{t}(x,y)=0$ otherwise. The coordinate-wise marginal
generator of Section~\ref{sec:dag} is then exactly
\begin{equation}
\label{eq:Q-operator-form}
  Q_t=\lambda_t\sum_{e=1}^{\mathcal{D}}\big(P^{e}_{t}-I\big),
  \qquad
  U^{\theta}_{t,r}=\mu_{t,r}\sum_{d=1}^{\mathcal{D}}\big(\tilde P^{d,\theta}_{t,r}-I\big),
\end{equation}
with $\tilde P^{d,\theta}_{t,r}(x,x^{d\to s})=\tilde p^{d}_{t,r}(s|x)$. Indeed, for
$z^{d}\neq x^{d}$ the $(x,x^{d\to z^{d}})$ entry of \Eqref{eq:Q-operator-form} is
$\lambda_t p_{1|t}(z^{d}|x)$, the diagonal is
$\lambda_t\sum_{e}(p^{e}_{1|t}(x^{e}|x)-1)$, and all entries at Hamming distance $\ge2$
vanish; in particular $Q_t$ is \emph{single-site supported}.

MeanFlow's identity is obtained by differentiating $(t-r)u(x_t,t,r)$ \emph{along the
flow}, which produces the material derivative $\frac{d}{dt}=\partial_t+v\cdot\nabla_x$.
The correct discrete counterpart is supplied by Dynkin's formula. Let $(X_t)$ be the CTMC
with generator $Q_t$ and let $f_t$ be a time-dependent function on
$\mathcal{S}^{\mathcal{D}}$, viewed as a column vector, which is $C^1$ in $t$. Then
\[
  M_t \;\triangleq\; f_t(X_t)-\int_{0}^{t}\big(\partial_s f_s+Q_s f_s\big)(X_s)\,ds
\]
is a martingale, and consequently
$\frac{d}{dt}\mathbb{E}[f_t(X_t)]=\mathbb{E}[(\partial_t f_t+Q_t f_t)(X_t)]$. We therefore
define the \emph{material derivative} along the CTMC as
\begin{equation}
\label{eq:material-derivative}
  \frac{D}{Dt}\;\triangleq\;\partial_t+Q_t\,\cdot,
\end{equation}
where $Q_t\,\cdot$ denotes left multiplication (applied columnwise for matrix-valued
fields). This is the discrete analogue of $v\cdot\nabla_x$: the generator \emph{is} the
directional derivative along the process.

\subsection{Proof of Proposition \ref{prop:agi}}\label{ap:proof-prop1}

We first record the basic properties of $U_{t,r}=\frac{P_{t\to r}-I}{r-t}$ defined in
\Eqref{eq:U-def}.

\begin{proposition}[Basic properties of the average generator]
\label{prop:U-properties}
For $0\le t<r\le1$:
\emph{(a)} $U_{t,r}(x,z)\ge0$ for all $z\neq x$ and $\sum_{z}U_{t,r}(x,z)=0$, i.e.\
$U_{t,r}$ is itself a valid generator;
\emph{(b)} $\delta_x^{\top}(I+(r-t)U_{t,r})=p_{r|t}(\cdot|x)$ for every $x$, i.e.\ a
single Euler step of size $r-t$ reproduces the exact transition law;
\emph{(c)} $\lim_{r\to t}U_{t,r}=Q_t$.
\end{proposition}

\begin{proof}
(a) $P_{t\to r}$ is row-stochastic, so its off-diagonal entries are nonnegative and its
rows sum to one; subtracting $I$ and dividing by $r-t>0$ preserves nonnegativity off the
diagonal and yields zero row sums. (b) Immediate from
$\delta_x^\top P_{t\to r}=p_{r|t}(\cdot|x)$. (c)
$\lim_{r\to t}\frac{P_{t\to r}-I}{r-t}=\partial_rP_{t\to r}|_{r=t}=Q_t$ by
\Eqref{eq:kolmogorov-fwd-bwd}.
\end{proof}

\begin{proof}[Proof of Proposition~\ref{prop:agi}]
Let $\Omega_{t,r}\triangleq(t-r)\,U_{t,r}=I-P_{t\to r}$. By the backward equation in
\Eqref{eq:kolmogorov-fwd-bwd},
\[
  \partial_t\Omega_{t,r}
  =-\partial_tP_{t\to r}
  =Q_t\,P_{t\to r}
  =Q_t\big(I-\Omega_{t,r}\big)
  =Q_t-Q_t\,\Omega_{t,r},
\]
that is $\partial_t\Omega_{t,r}+Q_t\Omega_{t,r}=Q_t$, which is exactly
$\frac{D}{Dt}[(t-r)U_{t,r}]=Q_t$ with the material derivative
\Eqref{eq:material-derivative}. Expanding the left-hand side with the product rule,
\[
  \partial_t\Omega_{t,r}=U_{t,r}+(t-r)\,\partial_tU_{t,r},
  \qquad
  Q_t\Omega_{t,r}=(t-r)\,Q_tU_{t,r},
\]
and substituting gives $U_{t,r}+(t-r)(\partial_tU_{t,r}+Q_tU_{t,r})=Q_t$, i.e.\
\Eqref{eq:meanflow_identity}. The boundary case $r=t$ reduces to $U_{t,t}=Q_t$,
consistent with Proposition~\ref{prop:U-properties}(c).
\end{proof}

\subsection{Marginal projection and the tractable \texorpdfstring{projected}{projected} transport term}
\label{ap:projection}

The exact $U_{t,r}$ of \Eqref{eq:U-def} is \emph{not} single-site supported (a
finite-time propagator moves many coordinates at once), whereas the model
\Eqref{eq:Q-operator-form} is. We therefore enforce \Eqref{eq:meanflow_identity} after
projecting onto per-coordinate marginals with $\Pi^{d}$ of \Eqref{eq:projection}, which is
precisely the information a coordinate-factorized sampler consumes.

\begin{proposition}[Properties of $\Pi^{d}$]
\label{prop:projection}
Let $z^{d}\neq x^{d}$. Then \emph{(a)} $\Pi^{d}$ is linear and commutes with $\partial_t$;
\emph{(b)} if $M$ is single-site supported then $\Pi^{d}[M](x,z^{d})=M(x,x^{d\to z^{d}})$,
in particular $\Pi^{d}[Q_t](x,z^{d})=\lambda_t p_{1|t}(z^{d}|x)$ and
$\Pi^{d}[U^{\theta}_{t,r}](x,z^{d})=\mu_{t,r}\tilde p_{t,r}(z^{d}|x)$;
\emph{(c)} $\Pi^{d}[I](x,z^{d})=0$ and
$\Pi^{d}[P^{e}_{t}](x,z^{d})=\delta_{ed}\,p_{1|t}(z^{d}|x)$;
\emph{(d)} $\Pi^{d}[U_{t,r}](x,z^{d})=p^{d}_{r|t}(z^{d}|x)/(r-t)$, so that matching
$\Pi^{d}[U^{\theta}_{t,r}]=\Pi^{d}[U_{t,r}]$ is exactly the condition that the
coordinate-factorized one-step sampler reproduce the correct per-coordinate marginals of
the true jump kernel.
\end{proposition}

\begin{proof}
(a) is immediate. (b) The only $y$ within Hamming distance one of $x$ with
$y^{d}=z^{d}\neq x^{d}$ is $y=x^{d\to z^{d}}$. (c) Apply (b) to $I$ (supported on $y=x$,
which has $y^{d}=x^{d}\neq z^{d}$) and to $P^{e}_t$ (supported on $y=x^{e\to s}$, which
requires $e=d$ and $s=z^{d}$). (d) Follows from \Eqref{eq:U-def} and
\Eqref{eq:projection}.
\end{proof}

Applying $\Pi^{d}$ to \Eqref{eq:meanflow_identity} and using
Proposition~\ref{prop:projection}(a)--(b) gives the projected identity we train against:
for every $d$ and every $z^{d}\neq x^{d}$,
\begin{equation}
\label{eq:marginal-identity}
  \Pi^{d}\big[U_{t,r}\big](x,z^{d})
  =\lambda_t\,p_{1|t}(z^{d}|x)
  -(t-r)\Big(\partial_t\,\Pi^{d}\big[U_{t,r}\big](x,z^{d})
  +\Pi^{d}\big[Q_tU_{t,r}\big](x,z^{d})\Big).
\end{equation}
Only the projected transport term couples different coordinates. This is not because $\Pi^{d}$ fails to commute with $Q_t$: since $\Pi^{d}[M]=ME_d$ acts on the target state, it commutes with left multiplication, $\Pi^{d}[Q_tM]=Q_t\,\Pi^{d}[M]$, as made explicit in \Eqref{eq:closure}. Rather, $Q_t$ acts on the source state, so $Q_t\,\Pi^{d}[U_{t,r}]$ evaluates $\Pi^{d}[U_{t,r}]$ at the states $x^{e\to s}$ reached by a jump of any coordinate $e$, not only of $d$. The next proposition evaluates it in closed-form for the
single-site model, which is what makes the objective tractable and proves
\Eqref{eq:commutator}.

\begin{proposition}[Tractable projected transport term]
\label{prop:tractable-QU}
Let $Q_t$ and $U^{\theta}_{t,r}$ be as in \Eqref{eq:Q-operator-form}. Then for every $d$
and every $z^{d}\neq x^{d}$,
\[
  \Pi^{d}\big[Q_tU^{\theta}_{t,r}\big](x,z^{d})
  =\lambda_t\,\mu_{t,r}\Big[\,C^{d}_{\theta}(x,z^{d})-p_{1|t}(z^{d}\mid x)\Big],
\]
with $C^{d}_{\theta}$ as in \Eqref{eq:C-term}.
\end{proposition}

\begin{proof}
Expanding the product of the two expressions in \Eqref{eq:Q-operator-form},
\[
  Q_tU^{\theta}_{t,r}
  =\lambda_t\mu_{t,r}\sum_{e=1}^{\mathcal{D}}\sum_{d'=1}^{\mathcal{D}}
  \Big(P^{e}\tilde P^{d'}-P^{e}-\tilde P^{d'}+I\Big),
\]
where we drop the sub/superscripts $t,r,\theta$ for readability. Fix $d$ and
$z^{d}\neq x^{d}$ and apply $\Pi^{d}$ term by term.

\textbf{Step 1.} By Proposition~\ref{prop:projection}(c), $\Pi^{d}[I]=0$,
$\sum_{e,d'}\Pi^{d}[P^{e}](x,z^d)=\mathcal{D}\,p_{1|t}(z^{d}|x)$ and
$\sum_{e,d'}\Pi^{d}[\tilde P^{d'}](x,z^d)=\mathcal{D}\,\tilde p_{t,r}(z^{d}|x)$, each sum
contributing a factor $\mathcal{D}$ from the free index.

\textbf{Step 2.} For the products, write
$(P^{e}\tilde P^{d'})(x,y)=\sum_{s}P^{e}(x,x^{e\to s})\,\tilde P^{d'}(x^{e\to s},y)$, so
that $y$ must be of the form $y=x^{e\to s,\,d'\to s'}$. Requiring $y^{d}=z^{d}\neq x^{d}$
forces $d\in\{e,d'\}$ with the corresponding new symbol equal to $z^{d}$. Two cases:
\emph{(i)} $d'=d$: then $s'=z^{d}$ and the remaining index $s$ is free and summed over by
$\Pi^{d}$, giving, for every $e$ (including $e=d$),
$\sum_{s}p^{e}_{1|t}(s|x)\tilde p_{t,r}(z^{d}|x^{e\to s})
=\mathbb{E}_{s\sim p^{e}_{1|t}(\cdot|x)}[\tilde p_{t,r}(z^{d}|x^{e\to s})]$.
\emph{(ii)} $d'\neq d$: then coordinate $d$ can only be set by the first factor, so $e=d$
and $s=z^{d}$; the index $s'$ is free and summed by $\Pi^{d}$, and since $\tilde P^{d'}$
is row-stochastic, $\sum_{s'}\tilde p_{t,r}(s'|x^{d\to z^{d}})=1$. Each of the
$\mathcal{D}-1$ values of $d'$ contributes $p_{1|t}(z^{d}|x)$. Summing,
\[
\sum_{e,d'}\Pi^{d}\big[P^{e}\tilde P^{d'}\big](x,z^{d})
=\sum_{e}\mathbb{E}_{s\sim p^{e}_{1|t}(\cdot|x)}\big[\tilde p_{t,r}(z^{d}|x^{e\to s})\big]
+(\mathcal{D}-1)\,p_{1|t}(z^{d}|x).
\]

\textbf{Step 3.} Combining Steps 1 and 2,
\[
  \Pi^{d}\big[Q_tU^{\theta}\big](x,z^{d})
  =\lambda_t\mu_{t,r}\Big[
  \sum_{e}\mathbb{E}_{s\sim p^{e}_{1|t}(\cdot|x)}\big[\tilde p_{t,r}(z^{d}|x^{e\to s})\big]
  -\mathcal{D}\,\tilde p_{t,r}(z^{d}|x)
  -p_{1|t}(z^{d}|x)\Big],
\]
which is the claim after grouping the $\mathcal{D}$ copies of $\tilde p_{t,r}(z^{d}|x)$
into the $\mathcal{D}$ summands of \Eqref{eq:C-term}.
\end{proof}

\paragraph{What the objective learns.}

The projected identity \Eqref{eq:marginal-identity} is closed in the per-coordinate
marginals. For a matrix $M$ with zero row sums, extend $\Pi^{d}[M](x,\cdot)$ to $z^d=x^d$ by
$\Pi^{d}[M](x,x^d)\triangleq\sum_{y:\,y^d=x^d}M(x,y)=-\sum_{s\neq x^d}\Pi^{d}[M](x,s)$.
Exchanging the sums in \Eqref{eq:projection} gives
\begin{equation}
\label{eq:closure}
  \Pi^{d}\big[Q_tM\big](x,z^{d})
  =\sum_{y:\,y^{d}=z^{d}}\sum_{x'}Q_t(x,x')\,M(x',y)
  =\sum_{x'}Q_t(x,x')\,\Pi^{d}[M](x',z^{d}),
\end{equation}
so the projected transport term depends on $M$ only through its off-diagonal per-coordinate marginals
at the states $x'$ reachable from $x$ in one jump; \Eqref{eq:transport-general} is
\Eqref{eq:closure} written out for a single-site $M$. In particular
$\Pi^{d}[Q_tU^\theta_{t,r}]=\Pi^{d}[Q_tU_{t,r}]$ whenever
$\Pi^{d}[U^\theta_{t,r}]=\Pi^{d}[U_{t,r}]$ for all $d$ and all states, so replacing $U_{t,r}$
by the single-site model in the projected transport term is not an approximation.

\begin{proposition}[Fixed point of the projected identity]
\label{prop:fixed-point}
Fix $r$ and assume that $t\mapsto Q_t$ is continuous on $[t_0,r]$. Let $V_t(x,z^d)$,
$z^d\neq x^d$, be per-coordinate rates that are $C^1$ in $t$ and bounded on $[t_0,r)$,
extended to $z^d=x^d$ by zero row sums, and suppose that for all $t\in[t_0,r)$, $x$, $d$
and $z^d\neq x^d$,
\begin{equation}
\label{eq:projected-fp}
  V_t(x,z^{d})
  =\Pi^{d}[Q_t](x,z^{d})
  -(t-r)\Big(\partial_tV_t(x,z^{d})+\sum_{x'}Q_t(x,x')\,V_t(x',z^{d})\Big).
\end{equation}
Then $V_t=\Pi^{d}[U_{t,r}]$, i.e.\ $V_t(x,z^d)=p^{d}_{r|t}(z^d|x)/(r-t)$.
\end{proposition}

\begin{proof}
By \Eqref{eq:closure}, \Eqref{eq:marginal-identity} states that $\Pi^{d}[U_{t,r}]$, which is
bounded by Proposition~\ref{prop:U-properties}(c), satisfies \Eqref{eq:projected-fp}. For
uniqueness, let $W_t\triangleq(t-r)V_t$. Since $\partial_tW_t=V_t+(t-r)\partial_tV_t$,
\Eqref{eq:projected-fp} is equivalent to the linear system
$\partial_tW_t=\Pi^{d}[Q_t]-\mathcal{Q}_tW_t$ with
$(\mathcal{Q}_tW)(x,z^d)\triangleq\sum_{x'}Q_t(x,x')W(x',z^d)$, in the finitely many
off-diagonal entries of $W_t$ (the diagonal ones being fixed by zero row sums). Boundedness
of $V_t$ gives $W_t\to0$ as $t\to r$. A linear ODE with continuous coefficients has a
unique solution with a prescribed terminal value, so $W_t$, and hence $V_t$, is unique.
\end{proof}

At the population optimum of \Eqref{eq:v-loss}, a model that can match the marginal rate
satisfies $Q^\theta_{t,r}(x_t,z^d)=\Pi^{d}[Q_t](x_t,z^d)$, the minimizer of the Bregman
divergence to the conditional rate. By \Eqref{eq:Q-reparam} and \Eqref{eq:closure} this is
\Eqref{eq:projected-fp} for $V_t=\Pi^{d}[U^\theta_{t,r}]$, so
$\Pi^{d}[U^\theta_{t,r}]=\Pi^{d}[U_{t,r}]$: the objective learns the per-coordinate
marginals of the exact finite-step kernel. Regressing $\Pi^{d}[U^\theta_{t,r}]$ directly
onto $\Pi^{d}[U_{t,r}]$ has the same solution, but requires samples of the kernel over
$[t,r]$, i.e.\ multi-step trajectories of a teacher; the projected identity only requires
the instantaneous rate and single-coordinate perturbations of the current state.

\subsection{The masked path}
\label{ap:masked}

For the masked source $p_0^d=\delta_{\mathbf{m}}$, with $\mathbf{m}$ the mask token and
$\mathcal{M}(x)$ the set of masked coordinates of $x$, the posterior and the network satisfy
the carry-over property $p^{e}_{1|t}(\cdot|x)=\tilde p^{e}_{t,r}(\cdot|x)=\delta_{x^{e}}$ for
$e\notin\mathcal{M}(x)$, and neither predicts $\mathbf{m}$. If $e\notin\mathcal{M}(x)$ then
$x^{e\to s}=x$ almost surely and the $e$-th summand of \Eqref{eq:C-term} vanishes; if
$e=d\in\mathcal{M}(x)$, carry-over gives $\tilde p^{d}_{t,r}(\cdot|x^{d\to s})=\delta_{s}$ and
$\mathbb{E}_{s}[\delta_{s}(z^{d})]=p_{1|t}(z^{d}|x)$. Hence
\begin{equation}\label{eq:C-masked}
C^{d}_{\theta}(x_t,z^{d})=\big(p_{1|t}(z^{d}|x_t)-\tilde p_{t,r}(z^{d}|x_t)\big)
+C^{d}_{\mathrm{cross}}(x_t,z^{d}),
\end{equation}
with $C^{d}_{\mathrm{cross}}(x_t,z^{d})\triangleq\sum_{e\in\mathcal{M}(x_t)
\setminus\{d\}}\big(\mathbb{E}_{s\sim p^{e}_{1|t}(\cdot|x_t)}[\tilde p_{t,r}(z^{d}|
x_t^{e\to s})]-\tilde p_{t,r}(z^{d}|x_t)\big)$, so that $\Gamma^{d}_{\theta}$ of \Eqref{eq:Gamma}, which contains the projected transport term, reduces to $C^{d}_{\mathrm{cross}}$ and
\Eqref{eq:At-tractable} becomes
\begin{equation}\label{eq:At-masked}
Q_{t,r}^\theta(x_t,z^{d})
= \lambda_t\,\tilde p_{t,r}(z^{d}\mid x_t)
+ \operatorname{sg}\Big\{(t-r)\,\mu_{t,r}\Big[
      \partial_t\tilde p_{t,r}(z^{d}\mid x_t)
      +\lambda_t\,C^{d}_{\mathrm{cross}}(x_t,z^{d})\Big]\Big\}.
\end{equation}
Since unmasked coordinates contribute nothing, $C^{d}_{\mathrm{cross}}$ is estimated with a
lower variance than \Eqref{eq:C-estimator} by drawing $e\sim\mathrm{Unif}(\mathcal{M}(x_t))$
and $s\sim p^{e}_{1|t}(\cdot|x_t)$ and setting
\begin{equation}\label{eq:C-estimator-masked}
  \widehat C^{d}_{\mathrm{cross}}(x_t,z^{d})
  =\big|\mathcal{M}(x_t)\big|\cdot\mathbbm{1}[e\neq d]\cdot
  \Big(\tilde p_{t,r}\big(z^{d}\mid x_t^{e\to s}\big)-\tilde p_{t,r}\big(z^{d}\mid x_t\big)\Big).
\end{equation}
For an unmasked $d$ the shifted probability is $\hat q^{d}_\theta=\delta_{x_t^d}$, and for a
masked $d$ it satisfies $\hat q^{d}_\theta(\mathbf{m}|x_t)=0$, so \Eqref{eq:loss-general}
reduces to a weighted cross-entropy over the masked coordinates,
\begin{equation}\label{eq:loss-masked}
  \mathcal{L}_{\tilde{p}}(\theta)
  =\mathbb{E}_{x_1,(t,r),x_t}\Bigg[\lambda_t\!\!\sum_{d\in\mathcal{M}(x_t)}\!\!
  -\log\bigg(
  \tilde p_{t,r}(x_1^{d}\mid x_t)
  +\operatorname{sg}\Big\{\tfrac{(t-r)\mu_{t,r}}{\lambda_t}
  \Big[\partial_t\tilde p_{t,r}(x_1^{d}\mid x_t)
  +\lambda_t\,C^{d}_{\mathrm{cross}}(x_t,x_1^{d})\Big]\Big\}
  \bigg)\Bigg]+\mathrm{const}.
\end{equation}
This is the objective we use on ImageNet.

\subsection{Sampling with the Discrete Average Generator}\label{ap:sampling}

\begin{algorithm}[t]
\caption{Sampling with the Discrete Average Generator}
\label{alg:dag-sampling}
\begin{algorithmic}[1]
\Require Network $H^\theta$, schedule $\kappa_t$, step count $K$, grid
  $0 = \tau_0 < \tau_1 < \cdots < \tau_K = 1$
\State $x_{\tau_0} \sim p_0$ \Comment{uniform noise; all-mask state for the masked source}
\For{$k = 0, \dots, K-1$}
  \State $t \gets \tau_k$, \; $r \gets \tau_{k+1}$
  \Comment{a single jump of size $r - t$}
  \State $\tilde p_{t,r}(\cdot \mid x_t) \gets H^\theta(x_t, t, r)$
  \Comment{one NFE, $\mathcal{D}\times|\mathcal{S}|$}
  \State $\displaystyle \omega_{t,r} \gets (r-t)\,\mu_{t,r}
    = \frac{\kappa_r - \kappa_t}{1 - \kappa_t}$
  \Comment{$\omega_{t,r}\in[0,1]$ by monotonicity of $\kappa$}
  \For{each coordinate $d = 1, \dots, \mathcal{D}$ \textbf{independently}}
    \State $x_r^d \gets z^d \sim \tilde p^{d}_{t,r}(\cdot \mid x_t)$
      with probability $\omega_{t,r}$, else $x_r^d \gets x_t^d$
    \Comment{masked source: unmasked $x_t^d$ carried over}
  \EndFor
\EndFor
\State \Return $x_{\tau_K}$
\end{algorithmic}
\end{algorithm}

\paragraph{Sampling.}
Algorithm~\ref{alg:dag-sampling} takes, independently for every coordinate, one Euler
step of the learned per-coordinate average generator; its one-step kernel is the product
\begin{equation}\label{eq:sampler-step}
  K^\theta_{t,r}(x_t,y)
  \;=\;\prod_{d=1}^{\mathcal{D}}\Big[
  \delta_{x_t^d}(y^d) + (r-t)\,\Pi^{d}\big[U^\theta_{t,r}\big](x_t,y^d)\Big],
\end{equation}
whose $d$-th factor is a distribution over $\mathcal{S}$, its entry at $y^d=x_t^d$ being
fixed by the zero row sum of $\Pi^{d}[U^\theta_{t,r}]$. This is not the full Euler step
$\delta_{x_t}^\top(I+(r-t)U^\theta_{t,r})$: since $U^\theta_{t,r}$ is single-site
supported, the latter moves at most one coordinate and, for large $\omega_{t,r}$, has a
negative diagonal entry.
In \Eqref{eq:sampler-step} the per-coordinate jump probability is
$(r-t)\,\Pi^{d}[U^\theta_{t,r}](x_t,z^d) = (r-t)\,\mu_{t,r}\,\tilde p_{t,r}(z^d|x_t)
= \omega_{t,r}\,\tilde p_{t,r}(z^d|x_t)$. Three properties follow directly from
Section~\ref{sec:dag} and are worth making explicit.

\emph{(i) Exact marginals.} By
Proposition~\ref{prop:projection}(d), if $\Pi^{d}[U^\theta_{t,r}]=\Pi^{d}[U_{t,r}]$ for every
$d$, then the $d$-th factor of \Eqref{eq:sampler-step} is exactly the $d$-th marginal of
$p_{r|t}(\cdot|x_t)$, for any step size: no discretisation error is incurred by the
per-coordinate jump. The remaining error of a single step is the dependence between the
coordinates under $p_{r|t}(\cdot|x_t)$, which a product kernel cannot represent.
Proposition~\ref{prop:fixed-point} shows that these exact marginals are the population target
of our objective.

\emph{(ii) Train/test consistency.} The factor $\mu_{t,r}$ appearing in
the sampler is the same one used to parameterise $U^\theta_{t,r}$ during training, so the
jump probability $\omega_{t,r}$ is by construction the quantity the training objective was
written for. Had we instead kept the instantaneous scale, $U^\theta = \lambda_t\tilde
p_{1|t}$, the sampler would use $(r-t)\lambda_t$ while the conditional path jumps with
probability $\frac{\kappa_r-\kappa_t}{1-\kappa_t}$; these agree only for the linear schedule
$\kappa_t = t$, and diverge as $O((r-t)^2\ddot\kappa_t)$ otherwise. Carrying $\mu_{t,r}$
explicitly removes this mismatch for any schedule.

\emph{(iii) Where the learning shows up.} The product in
\Eqref{eq:sampler-step} is over coordinates, so a single step cannot reproduce the
correlations of the true kernel $P_{t\to r}$. What the objective can improve is
therefore the per-coordinate marginals. We do not claim a guarantee on how much
exact marginals reduce the error of a $K$-step sampler; we measure this exactly in
Section~\ref{sec:sim}. Setting $r = t$ makes $\omega_{t,r}\to 0$ and recovers the
instantaneous rate, so Algorithm~\ref{alg:dag-sampling} degenerates gracefully to the
standard ancestral sampler as $K$ grows.

Finally, the network must be queried with the same $(t,r)$ conditioning it
saw in training: a $K$-step run with a uniform grid only ever presents step sizes
$r-t = 1/K$, so the training distribution over $(t,r)$ has to place mass on the step sizes
that will actually be used at test time.

\subsection{The uniform-state (DUO) process}
\label{ap:duo}

Everything above is stated for the mixture path of
\Eqref{eq:coordinate-wise conditional path}, whose conditional rate
$Q_t^{d}(x^{d},z^{d}|x_1^{d})=\lambda_t(\delta_{x_1^{d}}(z^{d})-\delta_{x^{d}}(z^{d}))$
does not depend on the source $p_0$. For a uniform source this is the CTMC in which a
single noise sample $x_0$ is drawn once and each coordinate either still shows $x_0$ or has
already switched to $x_1$. Uniform-state models in the DUO family \citep{sahoo2025the} use a
\emph{different} CTMC with the same marginals $p_{t|1}$: the noise is re-randomised at every
step, as in D3PM. This appendix records the parameterization we use on top of the
uniform-state checkpoints of our text experiments.

Throughout this subsection we follow the convention of DUO, in which
$\alpha_t$ decreases with $t$, $t=0$ is data, a step goes from $t$ to $r<t$, and
$c'\triangleq(1-\alpha_t)/|\mathcal{S}|$. Write $\alpha_{t|s}=\alpha_t/\alpha_s$.

\paragraph{Conditional and marginal rates.}
The conditional rate is obtained by differentiating the bridge
$q(x_s|x_t,x_1)\propto\big[\alpha_{t|s}\delta_{x_s}(x_t)+
\tfrac{1-\alpha_{t|s}}{|\mathcal{S}|}\big]\big[\alpha_s\delta_{x_1}(x_s)+
\tfrac{1-\alpha_s}{|\mathcal{S}|}\big]$ at $s=t$:
\begin{equation}
\label{eq:duo-cond-rate}
  Q_t(x_t,z\,|\,x_1)
  =\frac{-\dot\alpha_t}{\alpha_t|\mathcal{S}|}\cdot
   \frac{\alpha_t\mathbbm{1}[z=x_1]+c'}{\alpha_t\mathbbm{1}[x_1=x_t]+c'},
  \qquad z\neq x_t .
\end{equation}
Averaging \Eqref{eq:duo-cond-rate} over $x_1\sim p_{1|t}(\cdot|x_t)$ --- which by
Proposition~3.1 of \citet{campbell2024generative} is the marginal generator --- splits on $x_1\in\{z,x_t,\text{other}\}$
and yields the closed-form
$\frac{-\dot\alpha_t}{\alpha_t|\mathcal{S}|}\big[p_z\frac{\alpha_t+c'}{c'}
+p_{x_t}\frac{c'}{\alpha_t+c'}+(1-p_z-p_{x_t})\big]$. Unlike on the mixture path, this rate
is not proportional to the posterior mass $p_z$ of $z$, so the parameterization
$U^\theta_{t,r}=\mu_{t,r}\tilde p_{t,r}$ does not carry over.

\paragraph{Parameterisation.}
We take $U^{\theta}_{t,r}$ from the exact marginalization of the bridge,
\begin{equation}
\label{eq:duo-U}
  \Pi^{d}\big[U^{\theta}_{t,r}\big](x_t,z^{d})
  =\frac{1}{t-r}\,\mathbb{E}_{v\sim\tilde p^{d}_{t,r}(\cdot|x_t)}
  \big[q(x_r^d=z^{d}\,|\,x_t^d,\,x_1^d=v)\big],\qquad z^{d}\neq x_t^{d},
\end{equation}
which, because the bridge is a ratio of two functions affine in the one-hot encoding of
$v$ whose denominator takes only two values, has a closed-form costing $O(|\mathcal{S}|)$
per coordinate. The sampler draws each coordinate from the same kernel, so a single
per-coordinate Euler step of \Eqref{eq:duo-U} is exactly the sampler step.

\paragraph{The projected transport term in general form.}
Proposition~\ref{prop:tractable-QU} is a special case of an identity that holds
for \emph{any} pair of single-site matrices.  Let $q^{e}(x,\cdot)$ and
$u^{d}(x,\cdot)$ denote the coordinate-$e$ row of $Q_t$ and the coordinate-$d$ row
of $U^{\theta}_{t,r}$, both with zero row sum.  Then
\begin{align}
\label{eq:transport-general}
  \Pi^{d}\big[Q_tU^{\theta}_{t,r}\big](x,z^{d})
  &=\sum_{e=1}^{\mathcal{D}}\sum_{s} q^{e}(x,s)\,u^{d}\big(x^{e\to s},z^{d}\big)\notag\\
  &=\sum_{e=1}^{\mathcal{D}}\Lambda^{e}(x)
   \Big(\mathbb{E}_{s\sim\nu^{e}}\big[u^{d}(x^{e\to s},z^{d})\big]-u^{d}(x,z^{d})\Big),
\end{align}
with $\Lambda^{e}(x)=\sum_{s\neq x^{e}}q^{e}(x,s)$ the total outflow rate at
coordinate $e$ and $\nu^{e}=q^{e}/\Lambda^{e}$ the normalised jump law.  The
terms with $d'\neq d$ in the expansion vanish because $u^{d'}$ has zero row sum.
Substituting the mixture-path rows recovers $\lambda_t\mu_{t,r}(C^{d}_{\theta}-p_{1|t})$
of \Eqref{eq:commutator}; with the uniform-state rows, one draw
$e\sim\mathrm{Unif}(\{1,\dots,\mathcal{D}\})$, $s\sim\nu^{e}$ gives the unbiased estimate
$\mathcal{D}\,\Lambda^{e}(x)\big(u^{d}(x^{e\to s},z^{d})-u^{d}(x,z^{d})\big)$ at the cost of
a single forward pass. The objective is the Bregman divergence \Eqref{eq:v-loss}
between \Eqref{eq:duo-cond-rate} and the reparameterized rate, summed over all
$z^d\neq x_t^d$.

\paragraph{Sampling.}
The sampler draws each coordinate from the kernel \Eqref{eq:duo-U}, i.e.\ it takes one
per-coordinate Euler step of the learned average generator. Two conditioning details
matter for a model trained this way. First, the final noise-removal step
sets $\alpha_r=1$, i.e.\ $r=0$, so its effective step size is the whole of $t$
rather than the grid spacing; passing no step size at all there queries a conditioning
path the network never sees in training.  Second, a $K$-step run on a uniform grid
only ever presents $t-r=1/K$, so the training distribution over $(t,r)$ must place
mass there; the logit-normal design over $(t,r)$ used in our experiments places mass on the
step sizes $1/K$ of all budgets $K\le64$ that we evaluate.

Finally, nucleus filtering and softmax
temperature must be left at their neutral settings: both reshape the kernel after
the fact and would break the calibration that \Eqref{eq:duo-U} relies on.

\subsection{Proof of Training Objective for Discrete Average Generator}\label{ap:dag-objective-proof}
\begin{proof}
We show how the objective
\[
\mathcal{L}(\theta) = \mathbb{E}\Bigg[\sum_{d=1}^{\mathcal{D}}
\sum_{z^d \neq x_t^d}
D_F\!\bigg(Q_t(x_t,z^d|x_1) \;\bigg\|\;
\Pi^{d}\big[U_{t,r}^\theta\big](x_t,z^d)
+ \operatorname{sg}\!\Big(
(t-r)\Big(
\partial_t \Pi^{d}\big[U_{t,r}^\theta\big](x_t,z^d)
+ \Pi^{d}\big[Q_t\,U_{t,r}^\theta\big](x_t,z^d)
\Big)\Big)\bigg)\Bigg],
\]
with $\Pi^{d}[U_{t,r}^\theta](x_t, z^d) = \mu_{t,r}\,\tilde p_{t,r}(z^d|x_t)$,
reduces to \Eqref{eq:training-posterior_dmf} and then to \Eqref{eq:loss-general}.

\textbf{Step 1: Bregman divergence.}
For $F(x) = x\log x$, the Bregman divergence is
\[
D_F(a\|b) = a\log\frac{a}{b} - a + b.
\]

\textbf{Step 2: Conditional and target rates.}
For a single coordinate $d$, restricting to off-diagonal entries $z^d \neq x_t^d$,
the conditional rate is
\[
Q_{\mathrm{cond}}(z^d)
= Q_t^d(x_t^d, z^d | x_1^d)\,\mathbf{1}_{z^d \neq x_t^d}
= \lambda_t\,\mathbf{1}_{\{x_1^d \neq x_t^d\}}\,\delta_{x_1^d}(z^d).
\]
The target rate is
\[
Q_\theta(z^d)
= \Pi^{d}\big[U_{t,r}^\theta\big](x_t, z^d) + A_t(x_t, z^d)
= \mu_{t,r}\,\tilde p_{t,r}(z^d|x_t) + A_t(x_t, z^d),
\]
where
\[
A_t(x_t, z^d) \triangleq \operatorname{sg}\!\Big(
(t-r)\Big(
\partial_t \Pi^{d}\big[U_{t,r}^\theta\big](x_t,z^d)
+ \Pi^{d}\big[Q_t\,U_{t,r}^\theta\big](x_t,z^d)
\Big)\Big).
\]

\textbf{Step 3: Support of the conditional rate.}
The conditional rate has support only at $z^d = x_1^d$ when $x_1^d \neq x_t^d$,
and is identically zero otherwise. The analysis splits into two cases.

\textbf{Step 4: Log term ($x_1^d \neq x_t^d$).}
At $z^d = x_1^d$:
\[
Q_{\mathrm{cond}}(x_1^d)\log\frac{Q_{\mathrm{cond}}(x_1^d)}{Q_\theta(x_1^d)}
=\lambda_t\log\lambda_t-\lambda_t\log Q_\theta(x_1^d),
\]
whose first term does not depend on $\theta$.

\textbf{Step 5: Linear terms ($x_1^d \neq x_t^d$).}
The linear terms from $D_F$ at $z^d = x_1^d$ are
\[
-Q_{\mathrm{cond}}(x_1^d) + Q_\theta(x_1^d)
= -\lambda_t + Q_\theta(x_1^d).
\]
For all $z^d \neq x_1^d$ (and $z^d \neq x_t^d$),
$Q_{\mathrm{cond}}(z^d) = 0$, so $D_F(0\|Q_\theta(z^d)) = Q_\theta(z^d)$.
Summing over all $z^d \neq x_t^d$ gives
\begin{align*}
\sum_{z^d \neq x_t^d} Q_\theta(z^d)
&= \mu_{t,r}\big(1-\tilde p_{t,r}(x_t^d|x_t)\big)
+ \sum_{z^d \neq x_t^d} A_t(x_t, z^d).
\end{align*}
The last term cannot in general be rewritten as $-A_t(x_t,x_t^d)$: that
step would require $\sum_{z^d}A_t(x_t,z^d)=0$, whereas
$\sum_{z^d}\partial_t\Pi^{d}[U^\theta_{t,r}](x_t,z^d)=\partial_t\mu_{t,r}\neq0$. We
therefore keep the linear contribution as $\sum_{z^d\neq x_t^d}Q_\theta(z^d)$. For
the rate \Eqref{eq:At-tractable} with a normalized $\tilde p_{t,r}$ it takes a simple form:
since $\sum_{z^d}\tilde p_{t,r}(z^d|x_t)=1$ for all $t$ implies
$\sum_{z^d}\partial_t\tilde p_{t,r}(z^d|x_t)=0$, and each summand of $C^{d}_{\theta}$, as well
as $\tilde p_{t,r}-p_{1|t}$, is a difference of two normalized distributions, the
correction sums to zero over $z^d\in\mathcal{S}$ and
\[
\sum_{z^d \neq x_t^d} Q_\theta(z^d)
= \lambda_t\sum_{z^d \neq x_t^d}\hat q^{d}_{\theta}(z^d|x_t)
= \lambda_t\big(1-\hat q^{d}_{\theta}(x_t^d|x_t)\big),
\]
with $\hat q^{d}_{\theta}$ the shifted probability of Section~\ref{sec:tractable}. On
the masked path, $\hat q^{d}_\theta(\mathbf{m}|x_t)=0$ for masked $d$, so this is the
constant $\lambda_t$.

Combining the linear contributions in this case gives
\[
-Q_{\mathrm{cond}}(x_1^d) + \sum_{z^d \neq x_t^d} Q_\theta(z^d)
= -\lambda_t + \sum_{z^d \neq x_t^d} Q_\theta(z^d),
\]
where the leading $-\lambda_t$ is independent of $\theta$ and is absorbed
into the additive constant of the loss.

\textbf{Step 6: Case $x_1^d = x_t^d$.}
In this case $Q_{\mathrm{cond}}(z^d) = 0$ for all $z^d \neq x_t^d$, so
$D_F(0\|Q_\theta(z^d)) = Q_\theta(z^d)$ and the total contribution is
$\sum_{z^d \neq x_t^d} Q_\theta(z^d)$, i.e.\ the same linear expression as
in Step 5 with the logarithmic term absent and without the $-\lambda_t$ offset.

\textbf{Step 7: Combine terms.}
Combining both cases using the indicator $\delta_{x_1^d}(x_t^d)$, dropping the
$\theta$-independent terms (the $a\log a$ part of $D_F$ and the offset
$-\lambda_t(1-\delta_{x_1^d}(x_t^d))$), and summing over all coordinates
$d=1,\dots,\mathcal{D}$, we obtain
\begin{equation}\label{eq:training-posterior_dmf}
\mathcal{L}_{Q}(\theta)=\mathbb{E}_{x_1,(t,r),x_t}\Bigg[\sum_{d=1}^{\mathcal{D}}\bigg\{
-\big(1-\delta_{x_1^d}(x_t^d)\big)\,\lambda_t\,
\log Q_{t,r}^\theta(x_t, x_1^d)
+ \sum_{z^d \neq x_t^d} Q_{t,r}^\theta(x_t, z^d)
\bigg\}\Bigg] + \mathrm{const},
\end{equation}
where $Q_{t,r}^\theta(x_t,z^d)=Q_\theta(z^d)$ is the reparameterized
instantaneous rate of \Eqref{eq:Q-reparam}. Substituting \Eqref{eq:At-tractable}, writing
$\log Q^\theta_{t,r}(x_t,x_1^d)=\log\lambda_t+\log\hat q^{d}_\theta(x_1^d|x_t)$ and using the
linear term above yields \Eqref{eq:loss-general}, and on the masked path
\Eqref{eq:loss-masked}. This completes the proof.
\end{proof}

\section{Experimental Details}

\subsection{Experimental Setup of Simulations}\label{ap:sim-exact}

\paragraph{Data distribution}
We consider a data distribution in which every pair of coordinates is coupled, i.e.\ we
take the Potts law
$$
q(x)\;\propto\;\exp\Big(\beta\!\!\sum_{1\le i<j\le\mc{D}}\!\!\mathbbm{1}\big[x^{i}=x^{j}\big]
\;+\;\varepsilon_{x}\Big),
\qquad \varepsilon_{x}\sim0.3\,\mc{N}(0,1),
$$
where $\beta\ge0$ controls how strongly the coordinates are coupled, and the perturbations
$\varepsilon_{x}$ are drawn once and then held fixed so as to break the permutation
symmetry of the energy. We use $\beta=1.5$ and dimension $\mc{D}=4$.
The vocabulary size is
$4$.

\paragraph{Experimental setup and implementation details.}
The state space has $\abs{\mc{S}}^{\mc{D}}=256$ elements, so the target $q$, the posterior
$p_{1|t}$, the projected transport term \Eqref{eq:commutator} and the law of the $K$-step sampler are
all obtained by enumeration, i.e., no quantity we report is estimated from samples. We
parameterize the network as
MLPs with SiLU activation functions, consisting of four hidden layers with hidden
dimension 256, conditioned on Fourier features of $t$, $r$, $r-t$ and $-\log(1-t)$. We
first train the base denoiser with \Eqref{eq:training-posterior} using the Adam optimizer
with a learning rate of 2e-3 for 1.2e4 steps, and then train only the correction head with
\Eqref{eq:loss-general} for 3e3 steps with the base frozen, so that every
objective we compare shares an identical base denoiser.
The pairs $(t,r)$ are drawn from the evaluation
grids, with $25\%$ of the draws at $r=t$; the instantaneous posterior $p^{e}_{1|t}$ is
supplied by the frozen base denoiser. In the sampling stage, we evaluate the trained network at every state and
propagate the law of the sampler exactly, for a number of steps
$K\in\set{2,3,4,6,8,12,16,24,32}$, and record the total variation between $q$ and that
law. We run each setting with 3 seeds and report the mean.

\subsection{Implementation Details of OpenWebText and ImageNet}\label{ap:impl}
The
network is conditioned on $t$ and on the step size $r-t$. On OpenWebText, $\tau_t$ and
$\tau_r$ are the larger and the smaller of two logit-normal draws
$\sigma(\mathcal{N}(P_{\mathrm{mean}},P^2_{\mathrm{std}}))$ with
$(P_{\mathrm{mean}},P_{\mathrm{std}})=(0,1.5)$ and $(-1.0,1.5)$; on ImageNet, $\tau_t$ is the
arccos mask ratio and $\tau_r$ is a logit-normal draw with $(-3.0,2.0)$ clipped to
$\tau_r\le\tau_t$, which places more mass on the large steps of the $1$--$4$ step regime. A
fraction $1-\rho=0.25$ of the pairs is drawn with $r=t$, where the correction vanishes and
the objective reduces to \Eqref{eq:training-posterior}; this is what keeps the boundary
calibrated.
On ImageNet the sampling
hyperparameters of every method, ours and the baselines, are selected by a sweep and
we report the best configuration of each method; on OpenWebText the baseline numbers are
taken from \citet{Yoo2025ReDiRD}.

\subsection{Detailed values for Section \ref{sec:exper}}
Detailed values for simulations, OpenWebText, and ImageNet can be found in Tables \ref{tab:sim-full}, \ref{tab:openwebtext}, and \ref{tab:main_results}, respectively.

\begin{table}[h]
\centering
\caption{Total variation Potts model simulations. Mean over three seeds. These are
the values plotted in Figure~\ref{fig:simulations}.}
\label{tab:sim-full}
\small
\setlength{\tabcolsep}{4pt}
\begin{tabular}{lccccccccc}
\toprule
$K$ & $2$ & $3$ & $4$ & $6$ & $8$ & $12$ & $16$ & $24$ & $32$\\
\midrule
Baseline & $0.6020$ & $0.3707$ & $0.2456$ & $0.1351$ & $0.0907$ & $0.0508$ & $0.0356$ & $0.0225$ & $0.0136$\\
Ours & $\mathbf{0.5381}$ & $\mathbf{0.3170}$ & $\mathbf{0.1934}$ & $\mathbf{0.0957}$ & $\mathbf{0.0517}$ & $\mathbf{0.0181}$ & $\mathbf{0.0116}$ & $\mathbf{0.0080}$ & $\mathbf{0.0053}$\\
\bottomrule
\end{tabular}
\end{table}

\begin{table}[h]
\centering
\caption{OpenWebText results for Figure~\ref{fig:openweb}: generative perplexity measured with LLaMA 3.1-8B and entropy, as a function of the number of sampling steps $K$. Baseline numbers are taken from \citet{Yoo2025ReDiRD}. Best perplexity per column in bold.}
\label{tab:openwebtext}
\small
\setlength{\tabcolsep}{4pt}
\begin{tabular}{lccccc}
\toprule
& \multicolumn{5}{c}{Sampling steps $K$}\\
\cmidrule(lr){2-6}
Model & $4$ & $8$ & $16$ & $32$ & $64$\\
\midrule
\multicolumn{6}{l}{\emph{Generative perplexity} }\\
DUO+DCD & $97.22$ & $72.18$ & $54.82$ & $46.05$ & $42.38$\\
ReDi$^1$ & $79.35$ & $48.53$ & $37.13$ & $31.21$ & $29.12$\\
ReDi$^2$ & $69.01$ & $45.11$ & $36.42$ & $31.56$ & $29.85$\\
ReDi$^3$ & $53.24$ & $36.33$ & $30.34$ & $27.75$ & $26.78$\\
Di4C & $55.42$ & $36.52$ & $27.66$ & $23.04$ & $21.54$\\
Di4C + ReDi$^1$ & $\mathbf{47.50}$ & $30.92$ & $24.81$ & $21.88$ & $20.32$\\
Ours & $56.08$ & $\mathbf{27.59}$ & $\mathbf{22.92}$ & $\mathbf{21.35}$ & $\mathbf{16.97}$\\
\midrule
\multicolumn{6}{l}{\emph{Entropy}}\\
DUO+DCD & $4.82$ & $5.25$ & $5.35$ & $5.36$ & $5.35$\\
ReDi$^1$ & $5.43$ & $5.49$ & $5.48$ & $5.44$ & $5.42$\\
ReDi$^2$ & $5.43$ & $5.44$ & $5.44$ & $5.41$ & $5.39$\\
ReDi$^3$ & $5.33$ & $5.30$ & $5.28$ & $5.26$ & $5.25$\\
Di4C & $4.81$ & $5.02$ & $5.07$ & $5.05$ & $5.03$\\
Di4C + ReDi$^1$ & $4.96$ & $5.07$ & $5.08$ & $5.09$ & $5.07$\\
Ours & $4.96$ & $5.06$ & $5.09$ & $5.18$ & $5.02$\\
\bottomrule
\end{tabular}
\end{table}

\begin{table}[h]
\centering
\caption{
ImageNet $256\times256$ with four network evaluations. Every method is reported at the best sampling configuration found by a
sweep. Best value per column in bold.
}
\label{tab:main_results}

\begin{tabular}{lcc}
\toprule
Model & FID $\downarrow$ & IS $\uparrow$ \\
\midrule

MaskGIT \citep{chang2022maskgit}
    & 10.90 & 184 \\
\cmidrule(lr){1-3}
SDTT \citep{deschenaux2024auto}
    & 8.97 & 205 \\
Di4C \citep{hayakawa2025distillation}
    & \textbf{6.20} & 216 \\
ReDi$^1$ \citep{Yoo2025ReDiRD}
    & 7.58 & 228 \\
ReDi$^2$ \citep{Yoo2025ReDiRD}
    & 7.86 & 240 \\
Ours
    & 7.81 & \textbf{250} \\

\bottomrule
\end{tabular}
\end{table}

\section{Unconditional Text Generation Samples ($K = 64$)}
\label{app:samples_k64}

Below we provide unconditional $1{,}024$-token text samples generated by our OpenWebText
model using the exact-marginalization sampler of
\Eqref{eq:duo-U}--\Eqref{eq:sampler-step} at $K=64$ sampling steps and temperature $T=1.0$.
Across the evaluation batch, this configuration achieves the generative perplexity and entropy
reported for ``Ours'' at $K=64$ in Table~\ref{tab:openwebtext} (LLaMA-3.1-8B generative
perplexity $16.97$, entropy $5.02$).

\begin{tcolorbox}[breakable,colback=gray!5!white,colframe=gray!75!black,title=\textbf{Sample 1: Energy Policy \& Regulatory Hearing ($K=64$, $T=1.0$)}]
\small\ttfamily
<|endoftext|> statement that he was "deeply disappointed" that the company's actions were "not in line with our public policy" during the regulatory process.\\[4pt]
"Everyone have a right to have concerns about energy security, and we do not believe that actions like this are in the vital national interest," Price said in the statement.\\[4pt]
Pamper E. Price said that the firm paid roughly \$14.1 million to lobbying groups in 2013, but a lawsuit filed by the Institute for Energy Policy Studies, the Washington-based Ceres' think group, that it won't be any more attention from the government.\\[4pt]
The company is expected to go before the House Committee on Commerce and Transportation before joining the House of Representatives this fall and hold a hearing. The White House said in a statement on Friday that it would continue to get the company's representatives to investigate its activities.\\[4pt]
"We are very disappointed this firm played an important role in efforts to implement critical energy security in our region and look forward to continuing to improve the decision being made to the American people of our region," the company said in the statement.\\[4pt]
The firms ties to politics have raised questions about the effectiveness of the Obama administration's political attacks on the industry, including the closure of a variety of plants in the United States and the hiring of about 1,200 employees from 35 individuals and nonprofit groups, including fossil fuel firms and environmental groups. The industry is heavily backed by special interests.\\[4pt]
One August, the firm's actions were part of the efforts to block the construction of the Dakota Access pipeline through Minnesota, the project being planned up in North Dakota carrying coal and crude oil.\\[4pt]
The Trump administration said the actions were related a March 2015 rule, which temporarily halted efforts to save large sections of the proposed pipeline. Protesters repeatedly blocked a portion of the proposed pipeline.\\[4pt]
“We look forward to hearing before the House today whether a few of these actions are in line with our national policy,” Price said. “There is no evidence, in fact, that these actions involving Goldstone Perkins are politically motivated and are not the policies we should be guided by.”<|endoftext|>
\end{tcolorbox}

\begin{tcolorbox}[breakable,colback=gray!5!white,colframe=gray!75!black,title=\textbf{Sample 2: Public Health Study \& International News ($K=64$, $T=1.0$)}]
\small\ttfamily
<|endoftext|> by reducing their daily intake of nutrients, such as fruits, vegetables, and nuts, could contribute to a large decrease in the number of displaced in China.\\[4pt]
The report, on the issue of food insecurity in China during the growing season, was authored by researchers from the Federal Institute of Public Affairs and the Department of Public Health of the University of Shanghai. The team, led by professors Zhou, U.S. Department of Agriculture, and Liu, also focused on the issue of the National Food Standards, which are mostly designed to combat the problem of hunger and food insecurity in China.\\[4pt]
A Professor of Public Health at the University, Wei Mao, who not involved as a senior author of the report, said that the administration decided not to enforce the National Standards "due to the misinformation or distortion of information that is currently available. Our results that excessively high-fat foods such as fruits, vegetables, can lead to high levels of blood sugar. These foods have increase the risk of heart disease and increase the risk of diabetes."\\[4pt]
The consumption of sugary drinks contributes to a high risk of food insecurity. Tens of thousand of the people currently without food in China is associated with the risk of diabetes mellitus which has nearly doubled in the past 25 years. And according to a UN monitor, the majority of the citizens of the world, Southeast Asia and East Asia, live on poor quality food conditions.\\[4pt]
Earlier this month, Shanghai Municipal Commission for Food Safety and Integrity, maker of the Standards, said that the government was carrying out an effort to improve the use of fresh vegetables in Beijing during the growing season. It's critical to understand different varieties of vegetables offer, such as fresh varieties, to encourage people to eat more vegetables, according to the report.\\[4pt]
The authors also explained that the global problem of food insecurity is that the government's largely neglected for decades, and which has hampered the efforts to combat the situation in China. As a result, they said, the need to reduce over-consumption and use of fresh food is a top priority in China.<|endoftext|>\\[4pt]
By James Mayr\\
London - World News staff\\[4pt]
British military and intelligence agencies on Friday called on the international community to prosecute the extremist organisations and anti-government forces to defeat in the name of peace and stability.\\[4pt]
The calls were made by at least two thirds of the leading members of British military, for the restoration of peace talks between Turkey and the rebels in northern Syria, and for an end to the war in Syria, which has left 250,000 people fleeing the country from abroad.<|endoftext|>
\end{tcolorbox}

\begin{tcolorbox}[breakable,colback=gray!5!white,colframe=gray!75!black,title=\textbf{Sample 3: International Diplomacy \& Economic Aid ($K=64$, $T=1.0$)}]
\small\ttfamily
<|endoftext|>United States has spoken out against what critics say is a poverty-ridden, expansionary and precarious economic situation in Egypt. Deputy Foreign Minister Rafa Mansour-Hahallah told reporters on Saturday the United States was “aware of a serious crisis with the country's dire situation.” “The U.S. is actively in the process to overcome this humanitarian crisis,” she said, after releasing a report saying that the Egyptian economy is in a “grave state of stagnation in the country.”\\[4pt]
The report said that the focus was on strengthening the security of the people and distributing programs to boost growth to end financial hardships for those living below the poverty level. She added that the situation of country is in a difficult position that any growth in activity could jeopardizing its security. Total oil production was cut by 5.2 million people in 2009 to just 3.1 million, according to the report.\\[4pt]
To mitigate the effects of the cuts, the federal subsidy is to freeze until the end of March. The proposed cuts, which will be needed to remove 3.5 million from the budget, where the central government has reserves, are only temporary. However, according to the report, the important thing is to improve the quality of life in the United States and “to restore job security and life changing services through the government's ambitious infrastructure plan,” administration officials said on Saturday.\\[4pt]
The plan is focused on creating jobs by 2011 and is meant to provide relief to the tens of millions of women, children and people living past the retirement age by bringing in full-time jobs. The State Department spokesman, Jay Carney, said that the federal government was focusing mainly to poverty-era programs and measures to help the affected people, most of whom are still facing job insecurity. “The government working hard on words, speaking by themselves.” Carney said, referring to comments published in The Washington Post. The U.S. government is determined to play a leading role in addressing this humanitarian situation, the official said. In addition to the call on the central government to take no further restrictions on oil and gas production, the United States however made it clear that the government would not place further restrictions. “The United States continues engaging with the people, and encouraging businesses to do as much as they can help restore the security of the people, as soon as now.” the United States Department said in a statement.<|endoftext|>
\end{tcolorbox}

\begin{tcolorbox}[breakable,colback=gray!5!white,colframe=gray!75!black,title=\textbf{Sample 4: Internet Openness, Privacy Law \& Climate Study ($K=64$, $T=1.0$)}]
\small\ttfamily
<|endoftext|> on June 7, 2011, in the city of Seattle, the US and Russia agreed to collaborate on three projects proposed by Russia, with the goal of protecting personal information and then they stole this information from the United States for the past 15 years. The most pressing thing they have to do is put pressure on the US government to pay for these projects because their practices have changed, and the peasants affected by this pipeline project are now before the eyes of the world," he said<|endoftext|>\\[6pt]
\textbf{Turkey}\\[4pt]
In January 2015, a citizen of Iran is six months pregnant in Turkey allegedly find themselves attempting to collect personal information of the Turkish government during a meeting of parliament of Turkey March 7, 2015. The Iranian was attempting to collect information about patient behavior, health and safety, and national security, breaking Turkish privacy laws and attempting to co-operate with Turkish authorities. No criminal charge was committed. The fine imposed on the person was determined by international law.\\[6pt]
\textbf{Germany}\\[4pt]
The United States, the United Kingdom, and the Netherlands support a variety of scientific, technological and social programs which benefit the people and environment of the world, and serve to be the engine of the global economy. This support is supported by the National Institutes of Health Sciences, the National Institutes of Health, the National Institute of Science and Technology, the Medical School of the University of Maryland, the School of Medicine and Public Health, the Massachusetts Institute of Technology, Energy and Environment, and Harvard University's School of Science, Space and Public Policy, and the Center for Environmental Science and Data Analysis at the University of Maryland.\\[6pt]
\textbf{United States}\\[4pt]
The United Nations Security Council reached an agreement on the openness of the Internet on July 11, 2011. One of the goals was to protecting national security and the openness of the Internet. On a global level, the Internet is probably one of the largest ecosystems in the world today. For example, the size of the web today is about 2.5 billion, the nature of the content has slightly through a variety of events, including melting Arctic sea ice. But over the course of the year, the surface extent of Arctic sea ice, increased steadily worldwide. A new study, shows glaciers now for about 55 percent of the world's surface ice, and about one-to-third of it is sea ice. (Changes to address climate change adaptation were made by three international experts, from the Institute for Media, Policy and Media Research, and from the World Economic Forum-WHO.)\\[4pt]
In June 2011, under Barack Obama, the United States and Russia reached a series of agreement which is designed to promote greater online openness on the Internet. The elements of this agreement: an international commitment to Internet openness, every online form of communication, full Web-wide Internet access, formal reviews, and certifying Internet providers for computers, mobile phones, and other forms of information.<|endoftext|>
\end{tcolorbox}

\begin{tcolorbox}[breakable,colback=gray!5!white,colframe=gray!75!black,title=\textbf{Sample 5: Sports Journalism: NHL \& Preseason Coverage ($K=64$, $T=1.0$)}]
\small\ttfamily
<|endoftext|>Yes.\\[4pt]
Muthyomun is an important member of the Chicago side which went through the very well and ended up in a third position in the W-L playoffs.\\[4pt]
The club had a favorable relationship with the Philadelphia Eagles and took over under the direction of Muthyomun in January. At the same time, the club now looking to reach at least 10th in W-L playoffs.\\[4pt]
Chicago head coach Antoine Bouchard is expected to visit the Chicago Blackhawks and make his preseason debut Wednesday, March 14 at the Chicago Blackhawks. In the past few months, he has been able to play as a key part of the lineup for the Chicago Blackhawks, according to the English-language website.\\[4pt]
Antoine Brouwer is a key part of the lineup for the Blackhawks and has already signed with the club most recently from the Anaheim Ducks. Anaheim had a great team last season to make the Stanley Cup Playoffs, and he is expected to be a part of the lineup for the first time since his injury in the Eastern Conference Playoffs.\\[4pt]
The Chicago Blackhawks had a great 2016 season and have a real chance at making another playoff run in the Eastern Conference. The club finished 23-16 in the regular season and was able to earn a third place spot in the WHLs first bid for the Stanley Cup by making a powerful shot that passed of his hand and his right foot bounced to the back of the net. Following hot start to the season he then scored 2 beautiful third and 2 goals in the third period against the Dallas Stars (4-2-1) and he was able to capitalize on it.\\[4pt]
He is considered one of the best players in the league and he has developed a good amount of confidence in his game. His right hand gives him good opportunities to make plays and boost his acceleration, for him to be able to hold on to his feet. While playing with his left foot down, he will have to make sure to find some movement at the end of the season with the Chicago Blackhawks.\\[4pt]
Brouwer returned to the Blackhawks after playing well in the American Hockey League. He was named the American Hockey League Most Outstanding Player of the Year and played a key part in the Chicago Blackhawks' victory over the Los Angeles Kings in the playoffs.<|endoftext|>
\end{tcolorbox}

\begin{tcolorbox}[breakable,colback=gray!5!white,colframe=gray!75!black,title=\textbf{Sample 6: US Politics \& Election Coverage ($K=64$, $T=1.0$)}]
\small\ttfamily
<|endoftext|> the most dangerous people in the federal government in 2016.\\[4pt]
The FBI confirmed they only interviewed Flynn on January 20.\\[4pt]
Washington (CNN) Matt Flynn reportedly not been interviewed by the FBI after the call was held, Director Comey reportedly chose not to answer all questions about it.\\[4pt]
Senate Democratic leaders said the action was the effective result of the official action and viewed the actions of President Trump as "a coup."\\[4pt]
"In order to keep Russia out of our country, we must continue to rely on the resolve of all members of Congress and support the American people," said Baucus, D-Mont., chairman of the council of The Church of Jesus Christ of Latter-day Saints.\\[4pt]
Vice President Joe Biden, at the center of hacking, said in a State of the Union address that he didn't immediately follow the department's emergency proclamation but it was an "appropriate decision."\\[4pt]
The accusations made their way through White House channels Saturday, said Frank Reich, a senior fellow at the liberal Center for American Progress and a member of the White House under President Donald Trump.\\[4pt]
"We have repeatedly repeated in the last days about the Trump campaign's contradictory statements about the use of chemical weapons against civilians," Reich told The New York Times. "Now, Mr. Biden's e-mails are a very damaging line of attack for the incoming administration."\\[4pt]
Senate Majority Whip Max Baucus, D-Mont., declined to comment about the allegations. On Saturday, he said, "We just have to see the results of the investigation conducted by prosecutors in our community."\\[4pt]
Former Attorney General Scott Baquet, who served under Robert M. Mueller as attorney general of the White House, has condemned the accusations.\\[4pt]
"It's difficult to say whether there's nothing else, but they certainly are there, and it's highly subjectively concluded that there's no merit in this case," he said.<|endoftext|>
\end{tcolorbox}

\subsection{CFG Ablation on ImageNet}\label{ap:imagenet-cfg}

\paragraph{Ablation Studies}
Sweeping classifier-free guidance (CFG) weight
$w\in\{1.5,2.0,3.0,4.0,5.0\}$ at $4$ and $8$ sampling steps (Figure~\ref{fig:cfg-ablation}), we
find that precision increases and recall decreases monotonically with $w$ at both step budgets,
that the best $4$-step FID ($7.81$) is attained at $w=3.0$, and that the optimal guidance weight
shifts down to $w=1.5$ (FID $6.90$) once step budget increases to $8$, where iterative
refinement already improves fidelity so a lower guidance weight suffices to preserve diversity.
Figure~\ref{fig:cfg-grid} shows the corresponding $4$-step samples across $w$, with the random
seed held fixed per class.

\begin{figure}[t]
    \centering
    \includegraphics[width=0.95\linewidth]{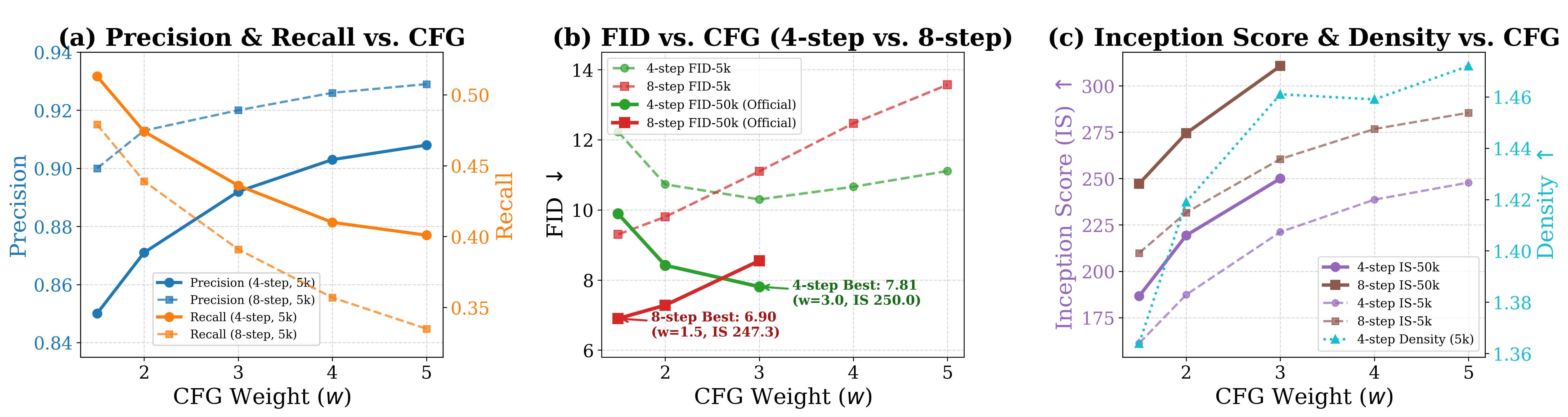}
    \caption{Ablation on the classifier-free guidance weight $w$ used to sample the training
    pairs, evaluated at $4$- and $8$-step generation on ImageNet $256\times256$ (50k samples). (a) Precision and recall vs.\ $w$.
    (b) FID vs.\ $w$, with the best configuration at each step budget annotated.
    (c) Inception score and density vs.\ $w$.}
    \label{fig:cfg-ablation}
\end{figure}

\begin{figure}[h!]
    \centering
    \includegraphics[width=0.85\linewidth]{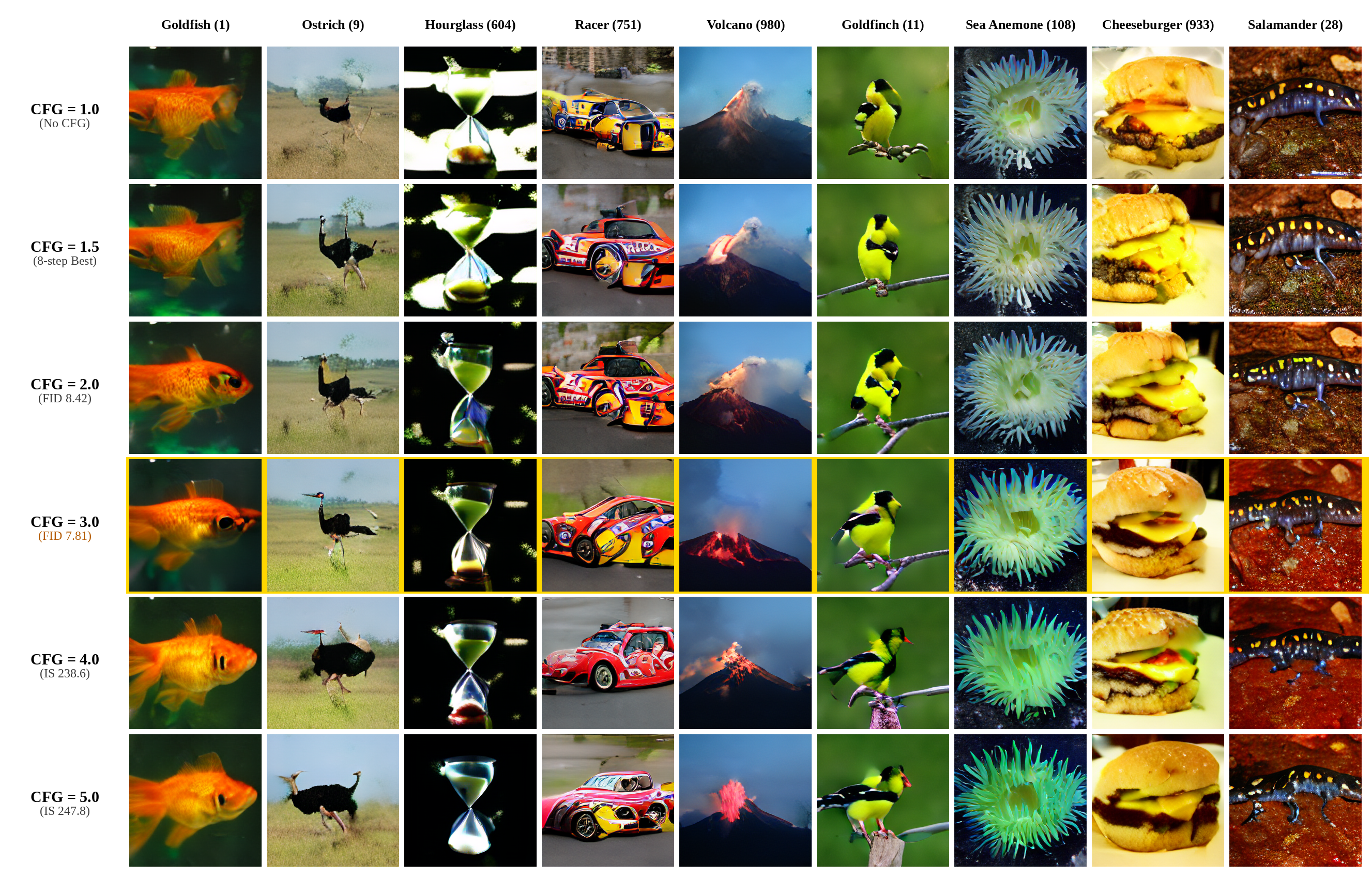}
    \caption{Qualitative effect of the CFG weight $w$ used to sample the training pairs, at
    $4$-step generation on ImageNet $256\times256$, for nine classes: goldfish, ostrich, hourglass, racer, volcano, goldfinch, sea anemone, cheeseburger and spotted salamander. The random seed is held
    fixed per class across $w\in\{1.0,1.5,2.0,3.0,4.0,5.0\}$, so only the effect of $w$
    varies; the optimal $4$-step setting ($w=3.0$, FID $7.81$) is highlighted.}
    \label{fig:cfg-grid}
\end{figure}

\subsection{Additional Qualitative Samples on ImageNet}\label{ap:imagenet-qual}

Figure~\ref{fig:imagenet-qual-appx} demonstrates the $4$-step qualitative comparison to four ImageNet classes (goldfinch, sea anemone, cheeseburger, spotted salamander)
following \citet{Yoo2025ReDiRD}.

\begin{figure}[h!]
    \centering
    \includegraphics[width=\linewidth]{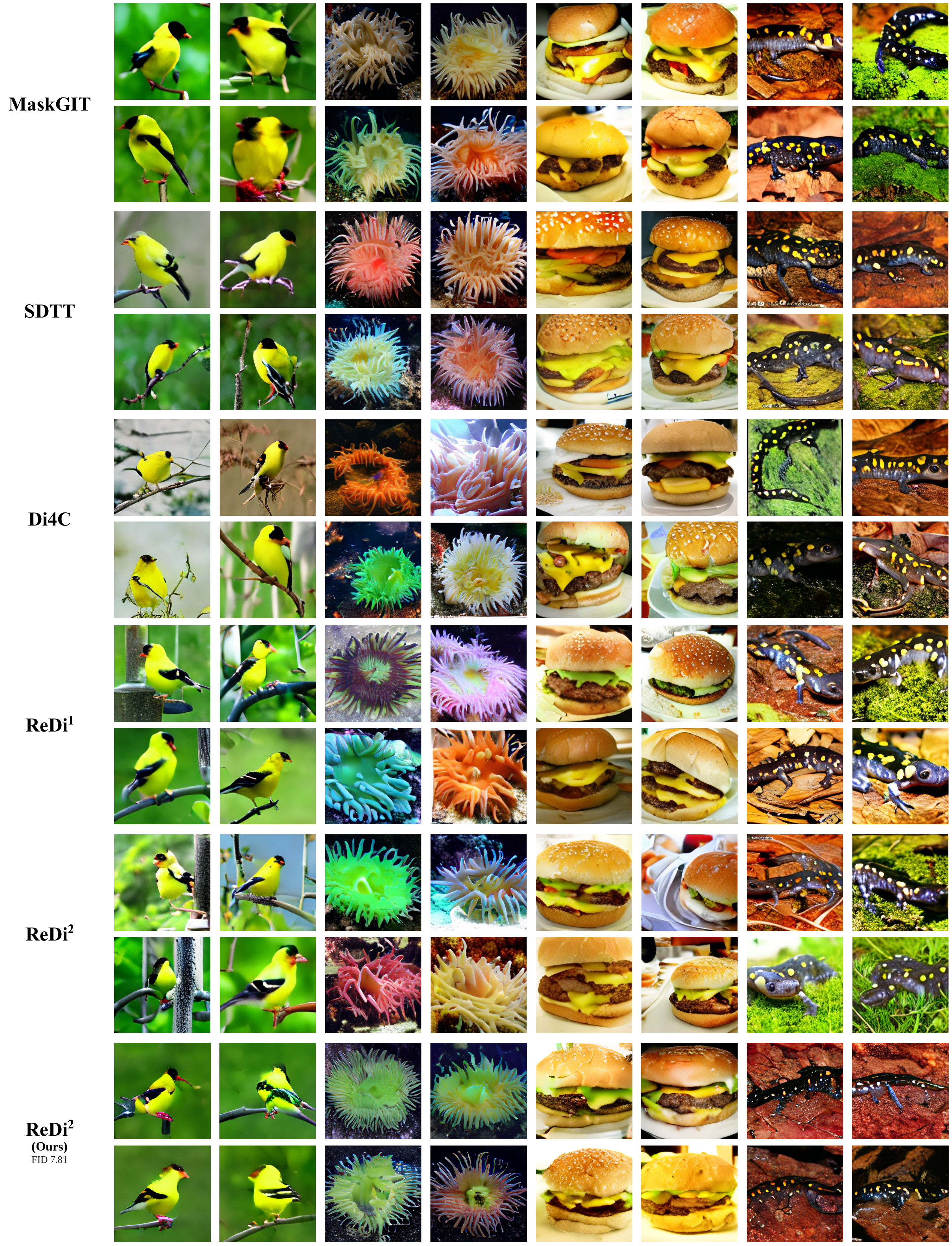}
    \caption{$4$-step generation results on ImageNet (classes goldfinch, sea
    anemone, cheeseburger, spotted salamander), two samples per class and model. Our method
    (bottom row, FID $7.81$) is appended below the baselines.}
    \label{fig:imagenet-qual-appx}
\end{figure}

\end{document}